\documentclass[letterpaper]{article} %
\usepackage{aaai24}  %
\usepackage{times}  %
\usepackage{helvet}  %
\usepackage{courier}  %
\usepackage[hyphens]{url}  %
\usepackage{graphicx} %
\usepackage{natbib}  %
\usepackage{caption} %
\usepackage{algorithm}
\usepackage{algorithmic}

\usepackage{newfloat}
\usepackage{listings}
\DeclareCaptionStyle{ruled}{labelfont=normalfont,labelsep=colon,strut=off} %
\floatstyle{ruled}
\newfloat{listing}{tb}{lst}{}
\floatname{listing}{Listing}
\title{Finite-Time Frequentist Regret Bounds of Multi-Agent Thompson Sampling on Sparse Hypergraphs}
\author{
    Tianyuan Jin\textsuperscript{\rm 1} \qquad
    Hao-Lun Hsu\textsuperscript{\rm 2} \qquad
    William Chang\textsuperscript{\rm 3} \qquad
    Pan Xu\textsuperscript{\rm 2}
}
\affiliations{
    \textsuperscript{\rm 1}National University of Singapore \\
    \textsuperscript{\rm 2}Duke University \\
    \textsuperscript{\rm 3}University of California, Los Angeles\\
    tianyuan@u.nus.edu, hao-lun.hsu@duke.edu,  chang314@g.ucla.edu, pan.xu@duke.edu
}

\usepackage{bibentry}
\usepackage{fancyvrb}
\usepackage{enumitem}
\usepackage[dvipsnames]{xcolor}
\usepackage[colorlinks,
            linkcolor=red,
            anchorcolor=blue,
            citecolor=blue, 
            pagebackref=true
           ]{hyperref}
\usepackage{mmll}

\renewcommand*{\backref}[1]{}
\renewcommand*{\backrefalt}[4]{%
\ifcase #1 %
   No citations.%
\or
   (p. #2.)%
\else
   (pp. #2.)%
\fi}%

\def \algname {$\epsilon$-\texttt{MATS}}
\def \localArmSize {A_{\text{loc}}}
\newcommand{\lev}{S_r}

\allowdisplaybreaks
\begin{document}
\maketitle

\begin{abstract}
We study the multi-agent multi-armed bandit (MAMAB) problem, where $m$ agents are factored into $\rho$ overlapping groups. Each group represents a hyperedge, forming a hypergraph over the agents. At each round of interaction, the learner pulls a joint arm (composed of individual arms for each agent) and receives a reward according to the hypergraph structure. Specifically, we assume there is a local reward for each hyperedge, and the reward of the joint arm is the sum of these local rewards. Previous work introduced the multi-agent Thompson sampling (MATS) algorithm \citep{verstraeten2020multiagent} and derived a Bayesian regret bound. However, it remains an open problem how to derive a frequentist regret bound for Thompson sampling in this multi-agent setting.
To address these issues, we propose an efficient variant of MATS, the $\epsilon$-exploring Multi-Agent Thompson Sampling ($\epsilon$-\texttt{MATS}) algorithm, which performs MATS exploration with probability $\epsilon$ while adopts a greedy policy otherwise. We prove that $\epsilon$-\texttt{MATS} achieves a worst-case frequentist regret bound that is sublinear in both the time horizon and the local arm size. We also derive a lower bound for this setting, which implies our frequentist regret upper bound is optimal up to constant and logarithm terms, when the hypergraph is sufficiently sparse. Thorough experiments on standard MAMAB problems demonstrate the superior performance and the improved computational efficiency of $\epsilon$-\texttt{MATS} compared with existing algorithms in the same setting. 
\end{abstract}

\section{Introduction}
Reinforcement learning (RL) is a fundamental problem in machine learning, where an agent learns to make optimal decisions in an environment by trial and error. A specific instance of RL is the multi-armed bandit (MAB) problem, in which an agent must choose between a set of arms, and each of the arms has a random reward distribution. The agent's goal is to maximize its total reward over time.
In standard MAB problems, an agent is provided with a set of arms $[K]:= {1, 2,...,K}$, and each arm, when pulled, generates a reward following a $1$-subgaussian distribution with an unknown mean. The agent's objective is to maximize its overall rewards within a specified time frame.

We consider the  \textbf{m}ulti-\textbf{a}gent MAB (MAMAB) problem, where there are $m$ agents. At each round of the interaction, each agent chooses an arm from its own arm set $[K]$.  
We define the concatenation of these arms as the joint arm. The bandit learner aims to coordinate with all agents and choose joint arms that maximize the cumulative rewards obtained from pulling those joint arms. It is important to note that the size of the joint arm space is exponential in the number of agents, specifically $A = K^m$. This exponential growth poses computational challenges in coordination and arm selection. To address this issue,
it was proposed to factor all agents into $\rho$ possibly overlapping groups (see wind farm application), which forms a hypergraph over the agents with each agent representing a node and each group representing a hyperedge (an illustration can be found in \Cref{sec:preliminary} and \Cref{fig:graph}).  Instead of pulling the joint arm, the learner only needs to pull the local arms, where each pulled local arm is defined as the concatenation of the arms chosen by agents within the same group. We assume each group has $d$ agents, and thus the total number of local arms $\localArmSize $ is at most $\rho K^d$, which is much smaller than the number of joint arms when the groups are small. 
This approach gives rise to MAMAB problems with specific coordination graph structures, which have found practical applications in various domains such as traffic light control \cite{wiering2000multiagent}, warehouse commissioning \cite{claes2017decentralised}, and wind farm control \cite{gebraad2015maximum, verstraeten2019fleetwide}. 

We evaluate a learning strategy based on its cumulative rewards obtained by interacting with the environment for a total of $T$ rounds. This evaluation can be equivalently measured by calculating the regret of the strategy compared to an oracle algorithm that always selects the arm with the highest reward. Mathematically, the regret is defined as $R_T = T\mu^* - \mathbb{E}[\sum_{t=1}^T f({\bm{A}_t})]$, where $\mu^*$ is the mean of the optimal arm and $f({\bm{A}_t})$ represents the reward obtained when pulling the joint arm $\bm{A}_t$ at time $t$ according to the given strategy. The goal of the algorithm (or learner) is to coordinate with all agents to determine the joint arm to pull in order to minimize this regret.

Thompson sampling \citep{thompson1933likelihood}, introduced by Thompson in 1933, has emerged as an attractive algorithm for bandit problems. It is favored for its simplicity of implementation, good empirical performance, and strong theoretical guarantees \citep{chapelle2011empirical,agrawal2017nearoptimal,jin2021mots}. The key idea behind Thompson sampling is to sample reward estimates for each possible arm from a posterior distribution and select the arm with the highest estimated value for pulling. In the single-agent setting, Thompson sampling has been shown to achieve near-optimal regret with respect to the worst possible bandit instance \citep{agrawal2017nearoptimal}.
In the context of multi-agent MAB with a coordination graph, the MATS (Multi-Agent Thompson Sampling) algorithm was proposed by \citet{verstraeten2020multiagent}. Unlike traditional Thompson sampling, where estimated rewards are sampled for each joint arm, MATS samples rewards for each local arm. This approach reduces the computational complexity, particularly in cases where the coordination hypergraph is sparse. \citet{verstraeten2020multiagent} provided a Bayesian regret bound for MATS, which measures the average performance given the probability kernel of the environment. However, in practical scenarios, it may not always be feasible for the learner to possess knowledge or access to the probability kernel of the environment. In such cases, the frequentist regret bound, which measures the worst-case performance across all environments, is often considered. It is worth noting that a frequentist regret upper bound implies a Bayesian regret bound, but not vice versa. Deriving a frequentist regret bound for the MATS algorithm in the multi-agent MAB problem with a coordination hypergraph remains an open question.

There are several technical challenges in the analysis of the frequentist regret for MATS. The first challenge emerges when applying the regret analysis of single-agent Thompson Sampling \citep{agrawal2012analysis} to our context. This occurs due to a dependence issue among different joint arms. Although rewards for each local arm are independently drawn from their respective reward distributions, the average rewards of the joint arms might be influenced by the other joint arms when they share some local arms (see \Cref{sec:chall} for detailed discussion). As a result, it is difficult to analyze the distribution of the average reward of the optimal arm or apply any concentration/anti-concentration inequalities, while all existing frequentist regret analyses of Thompson sampling \citep{agrawal2012analysis,agrawal2017nearoptimal,jin2021mots,jin2022finitetime,korda2013thompson,kaufmann2012thompson} heavily rely on the specific form of the distribution of the average reward of the optimal arm. A naive method of removing the dependence involves maintaining a posterior distribution for each joint arm and updating the distribution only when this joint arm is pulled. However, this method could result in significant computational complexity and regret due to the large joint arm space. %

In this paper, we tackle the issue using two strategies: 1) We carefully partition the entire arm set into subsets, ensuring each arm within a subset shares the same local arms with the optimal arm, and 2) We conduct a regret analysis at the level of local arms. Specifically, let $\bm{1}$ denote the optimal joint arm. We consider two events: 1) The local arm $\bm{1}^e$ of the optimal arm $\bm{1}$ is not underestimated, meaning the posterior sample of $\bm{1}^e$ is larger than $\bm{1}^e-\Delta/\rho$, and 2) The local arm $\bm{a}^e$ of the suboptimal arm $\bm{a}$ is not overestimated, meaning the posterior sample of $\bm{a}^e$ is lower than $\bm{a}^e-\Delta/\rho$. Crucially, these events ensure that the sum of posterior samples for any suboptimal joint arms is lower than the sum of posterior samples of $\bm{1}$, which leads to a lower regret. 

Another challenge in our local arm level analysis arises when we aim to establish a lower bound for the probability that the posterior sample of all local arms of $\bm{1}$ exceeds their means by $\Delta/\rho$. Leveraging the original Thompson Sampling analysis, we can establish this probability's lower bound as $(\Delta/{\rho})^{2\rho}$, leading to $(\rho/\Delta)^{2\rho}$ suboptimal arm pulls. In terms of worst-case regret, this amounts to $O\big(T^{\frac{2\rho-1}{2\rho}}\big)$. We improve this result by applying two innovative techniques (for ease of presentation, these are elaborated in full detail in \Cref{sec:chall}), reducing the number of pulls to $C^{\rho}$, where $C$ is a universal constant. Using these novel techniques, we are able to offer a $\sqrt{T}$-type worst-case regret.

\textbf{Main contributions.} We summarize our main contributions as follows.
\begin{itemize}[leftmargin=*,nosep]
 \item We propose the $\epsilon$-exploring Multi-Agent Thompson Sampling (\algname) algorithm, which only samples from the posterior distribution with probability $\epsilon$ and acts greedily with probability $1-\epsilon$. When $\epsilon=1$, our algorithm reduces to the MATS algorithm \citep{verstraeten2020multiagent}. When $\epsilon\ll 1$, our algorithm \algname\ only needs a small amount of exploration, and thus is much more computationally efficient than MATS in practice.
 
 \item We establish a frequentist regret bound for \algname\, in the order of $\Tilde{O}(\sqrt{(C/\epsilon)^\rho\localArmSize T})$, where $C$ is some universal constant and $\Tilde{O}(\cdot)$ ignores constant and logarithmic factors.  Here $\rho$ denotes the number of hyperedges,  $\localArmSize$ represents the total number of local arms, and $T$ is the time horizon. Remarkably, when $\epsilon=1$, our result also provides the first frequentist regret bound for MATS \citep{verstraeten2020multiagent}. Despite having $A$ joint arms, our regret bound grows as $O(\sqrt{\localArmSize})$, which is much smaller than the total number of joint arms.
 
 \item We also derive an $\Omega(\sqrt{\localArmSize T}/\rho)$ lower bound for the worst-case regret in our setting, showing that \algname\, is minimax optimal in terms of the total number of local arms and the time horizon when the the number of groups $\rho$ is small. When $\epsilon=1$, we show that the regret bound of MATS has an unavoidable $C^{\rho}$ dependence on the number of groups $\rho$, which further implies that the regret bound $\tilde{O}(\sqrt{(C/\epsilon)^{\rho}\localArmSize T})$ is optimal in terms of the number of groups as well.

 \item We further conduct extensive experiments on various MAMAB problems, including the Bernoulli 0101-Chain, the Poisson 0101-Chain, the Gem Mining problem, and the Wind Farm Control problem \citep{roijers2015computing,bargiacchi2018learning, verstraeten2020multiagent}. Through empirical evaluation, we demonstrate that the regret of \algname\, can be significantly lower compared to MATS as $\epsilon$ decreases, outperforming existing methods in the same setting. We also find that \algname\, exhibits improves computational efficiency compared to MATS.
\end{itemize}

\section{Preliminary and Background}\label{sec:preliminary}
In this section, we present the preliminary details of our setting. We also provide a notation table in \Cref{table:notation} for the convenience of our readers. 
We adopt the MAMAB (Multi-Agent Multi-Armed
Bandit) framework introduced by \citet{verstraeten2020multiagent}, where there are $m$ different agents, who are grouped into $\rho$ potentially overlapping groups. Each group can be represented as a hyperedge in a hypergraph, where the agents correspond to the nodes. Figure \ref{fig:graph} provides an example for easier visualization. During each round, every agent $i\in[m]$ selects an arm from their respective arm set $\mathcal{A}_i$, which is referred to as the "\emph{individual}" arm played by agent $i$. For simplicity, we assume that each agent $i$ has the same number of arms, denoted as $K=|\mathcal{A}_i|$. However, it is straightforward to extend the framework to accommodate varying numbers of arms $|\mathcal{A}_i|$. The arms chosen by all agents are concatenated to form a "\emph{joint}" arm denoted by $\ba$, which belongs to the set $\mathcal{A}_1 \times \cdots \times \mathcal{A}_m$. Consequently, the total number of joint arms is defined as $A:= |\mathcal{A}_1 \times \cdots \times \mathcal{A}_m|$.

We define a "\emph{local}" arm as the concatenation of individual arms for a specific group $e\in[\rho]$. In other words, if agents $i_1,\ldots ,i_d\in[m]$ form a hyperedge, then the local arm $\bm{a}^e \in \mathcal{A}_{i_1} \times \cdots \times \mathcal{A}_{i_d}$ represents the $d$-tuple of arms selected by these agents. We shall denote the set of local arms for group $e$ as $\mathcal{A}^e$. Let $\localArmSize$ be the total number of local arms. It is straightforward to see that $\localArmSize \leq \rho K^d$, with equality when the groups don't overlap. It is important to note that the arm space grows exponentially with the number of agents, leading to computational challenges in arm selection. To address this combinatorial complexity, we employ variable elimination techniques, which will be further explained in the subsequent sections.

In this paper, the global reward $f(\bm{a})$ associated with each joint arm $\bm{a}$ is decomposed into $\rho$ local rewards $f^e(\ba^e)$, where $\ba^e$ represents the local arm for group $e$. This decomposition takes advantage of the hypergraph structure. Specifically, for a given hypergraph with $\rho$ hyperedges, we have the relationship $f(\bm{a}) = \textstyle{\sum_{e=1}^\rho} f^e(\bm{a}^e)$.
The mean reward of a group $e$ is denoted as $\mu_{\bm{a}^e}= \mathbb{E}[f^e(\ba^e)]$. Consequently, the mean reward of a joint arm $\bm{a}$ is given by $\mu_{\bm{a}} = \textstyle{\sum_{e=1}^\rho} \mu_{\bm{a}^e} = \mathbb{E}[f(\ba)]$.
We assume the local rewards $f^e(\bm{a}^e)$ to be $1$-subgaussian, i.e. $\PP(|f^e(\bm{a}^e)| \geq \epsilon) \leq 2e^{-\epsilon^2}$. As a result, the global reward $f(\ba)$  is $\sqrt{\rho}$-subgaussian. %

Our objective is to maximize the expected cumulative global rewards obtained over a horizon of $T$ rounds of interaction with the environment. Without loss of generality, we assume that $\bm{1}$ is the optimal joint arm that yields the highest expected global reward. 
It is important to note that the goal is defined based on the performance of the best joint arm. In other words, even if a local arm $\bm{a}^e$ has a high local reward, it may not be selected frequently by an optimal policy if it is not part of joint arms with high mean rewards. 
To quantify the performance of a bandit strategy, we use the concept of regret denoted by $R_T$, defined as the expected difference between the cumulative rewards obtained by always selecting the optimal joint arm $\bm{1}$ and the actual rewards obtained by following a specific strategy. Mathematically, the regret $R_T$ is given by:
\begin{equation}
R_T = \mathbb{E}\big[\textstyle{\sum_{t=1}^T} (\mu_{\bm{1}} - f(\bm{A}_t))\big] = \textstyle{\sum_{t=1}^T }(\mu_{\bm{1}} - \mu_{\bm{A}_t}),
\end{equation}
where $\bm{A}_t$ represents the joint arm selected at round $t$. The regret captures the deviation from the cumulative rewards that would have been obtained if the optimal joint arm was chosen at each round. Minimizing regret is a key objective in designing effective strategies for the hypergraph MAMAB problem.
\begin{figure}
    \centering
    \includegraphics[width = .3\textwidth]{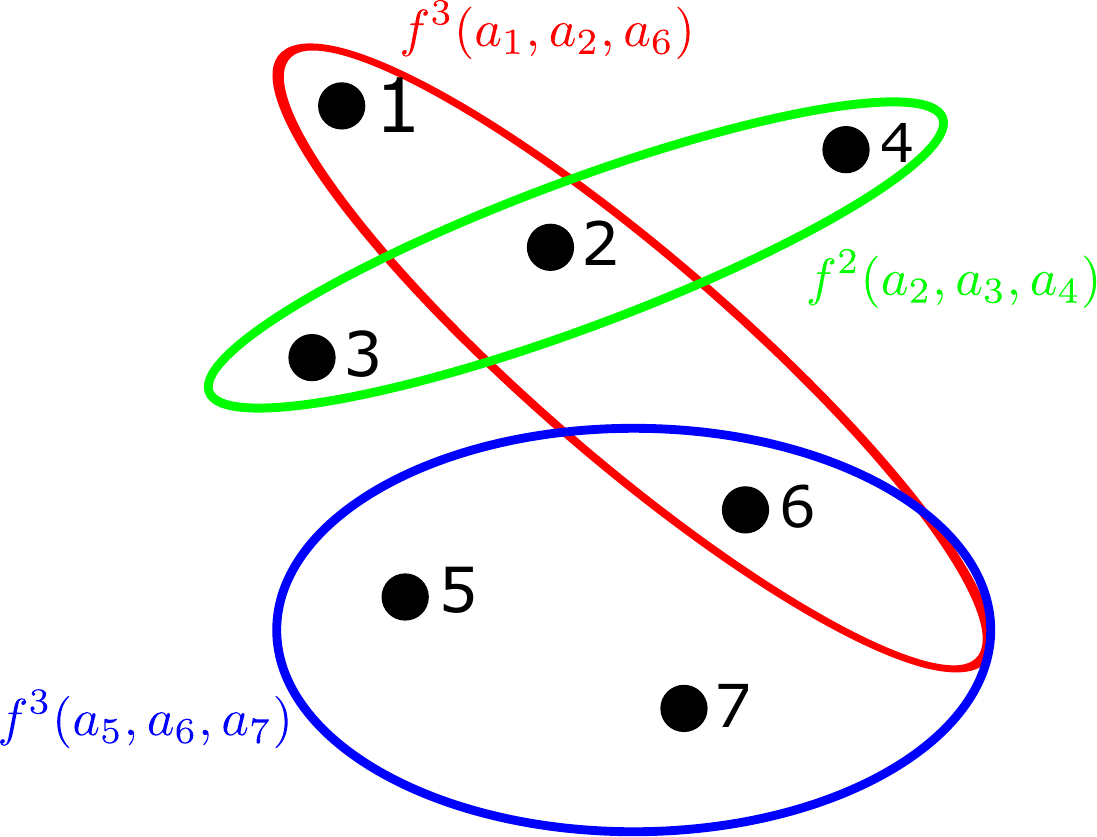}
    \caption{The hypergraph representation of a bandit environment with 8 agents and 3 groups. Each agent is represented by a vertex numbered by $\{1,2,\ldots 7\}$ and each group is represented by a hyperedge. In this case, there are three hyperedges with each with size $3$. Letting $a_i$ be the action taken by player $i$, the reward for the joint action $(a_1,\ldots,a_7)$ is decomposed as $f(a_1,a_2,\ldots,a_7)={\color{red}f^1(a_1,a_2,a_6)}+{\color{green}f^2(a_2, a_3, a_4)}+{\color{blue}f^3(a_5,a_6,a_7)}$, where $a_i$ is the individual arm picked by agent $i$. 
    }
    \label{fig:graph}
\end{figure}

\begin{remark}
    Although the results of our paper hold true regardless of the coordination hypergraph structure between the agents, they are most meaningful when the graph is sparse (i.e. the number of agents in each group is small). In particular, as there are $A$ joint arms, if one were to consider a different reward function for each arm, the regret and implementation complexity would be on the order of $A$. However, our results exploit the fact that there are only $\localArmSize $ local reward functions, and thus our regret bound is in terms of $\localArmSize $, which is much smaller than $A$ when the groups are small. 
\end{remark}

\section{The $\epsilon$-Exploring Multi-Agent Thompson Sampling Algorithm}\label{sec:algorithm}
In this section, we present the $\epsilon$-exploring Multi-Agent Thompson Sampling Algorithm (\algname), whose pseudo-code of \algname\ is displayed in \Cref{algo:MATS}. \algname\ is a combination of the MATS algorithm \cite{verstraeten2020multiagent} and a greedy policy. The idea of adding a greedy policy to Thompson Sampling was initially proposed in \citet{jin2022finitetime} and subsequently explored in \citet{jin2023thompson}. In \algname, at each round $t\in [T]$, similar to MATS,  \Cref{algo:MATS}
maintains a posterior distribution $\mathcal{N}(\widehat{\mu}_{\bm{a}^e}(t), \frac{c}{n_{\bm{a}^e}(t)})$ for each local arm $\bm{a}^e$, $e\in[\rho]$, where $\widehat{\mu}_{\bm{a}^e}(t)$ is the average reward of arm $\bm{a}^e$, $n_{\bm{a}^e}(t)$ represents the number of pulls of arm $\bm{a}^e$, and $c$ is a scaling parameter. Both MATS and \algname\ maintain estimated rewards $\theta_{\bm{a}^e}(t)$ for each local arm, and select the joint arm that yields the highest sum of estimated local rewards, i.e., $\bm{A}_t = \arg\max_{\bm{a}\in \cA} \sum_{e\in [\rho]}\theta_{\bm{a}^e}(t)$. After receiving the true rewards, the algorithms update the average reward $\widehat{\mu}_{\bm{a}^e}(t)$ and the number of pulls of $\bm{A}_t$ accordingly.

The difference between MATS and \algname\ lies in the way they construct the estimated rewards  $\theta_{\bm{a}^e}(t)$ for each local arm. In particular, 
MATS samples $\theta_{\bm{a}^e}(t)$ from the respective posterior distribution for local arm $\ba^e$. In contrast, the proposed \algname\ algorithm only  samples $\theta_{\bm{a}^e}(t)$ from the posterior distribution with a probability of $\epsilon$, and it directly sets $\theta_{\bm{a}^e}(t)=\widehat{\mu}_{\bm{a}^e}(t)$, i.e., as the empirical mean reward. Here $\epsilon\in(0,1]$ is a user-specified parameter that controls the level of exploration.  
For small values of $\epsilon$, \algname \ significantly reduces the level of exploration in MATS, which leads to improved computational efficiency.

\paragraph{$\epsilon$-Exploring.}
The idea of $\epsilon$-exploring is inspired from the recent work of \citet{jin2023thompson}.
We prove in the next section that \algname\, achieves the same order of finite-time regret bound as the MATS algorithm, even though it only needs to perform a small fraction of TS-type exploration. Furthermore, this algorithm runs faster since it doesn't have to sample each local arm from the Gaussian distribution as frequently as MATS. We also show that for specific applications the regret of \algname\ converges much faster than MATS and other algorithms in the same setting.

\paragraph{Variable Elimination.}
In Line~\ref{algline:arm_selection} of Algorithm \ref{algo:MATS}, \algname\ needs to find the joint arm $\ba$ that maximizes the sum of the estimated local rewards. However, this step can be computationally expensive if naively implemented, as it would require considering all possible joint arms, resulting in a complexity of $O(K^m)$ since the joint space size is $K^m$.  Following \cite{verstraeten2020multiagent}, we use variable elimination \cite{guestrin2001multiagent} to reduce this computation burden. 
The key idea behind variable elimination is to optimize over one agent at a time instead of summing all estimated local rewards for each joint arm and then performing the maximization. By doing so, we can significantly reduce the computational burden. To explain how variable elimination works, let us rewrite the maximum sum of the local estimates as follows:
{\footnotesize
\begin{align}
         &\max_{\bm{a}} f(\bm{a}) = \max_{\bm{a}}\sum_{e=1}^\rho   f^e (\bm{a}^e)=  \max_{\bm{a}} \bigg[ \underbrace{\sum_{\bm{a}^e \in \bm{a}: a_m \notin \bm{a}^e}f^e (\bm{a}^e)}_{I_1} \notag\\
         &\qquad + \underbrace{\max_{\bm{a}^e: a_m \in \bm{a}^e}  \sum_{\bm{a}^e \in \bm{a}:  a_m \in \bm{a}^e} f^e (\bm{a}^e)}_{I_2}\bigg],\label{eq:optimize}
\end{align}}
where $a_m$ represents an individual arm of agent $m$.  In Equation \eqref{eq:optimize}, we decompose the sum of the rewards into two cases based on the optimization variable $\ba$.
In $I_1$, we consider all the groups $\bm{a}^e$ that do not contain agent $m$. The maximization in this case is performed independently of the selection of individual arm $a_m$. Thus, the remaining agents can be optimized separately, resulting in a smaller optimization problem involving at most $m-1$ agents.
In $I_2$, we focus on the groups that contain agent $m$. Here, we aim to find the individual arm $a_m$ that maximizes the sum of the local rewards for the joint arms containing $a_m$. This sum depends on the individual arms of the other agents that share a group with agent $m$. After determining the optimal $a_m$, the rest of the maximization is performed independently on the remaining agents in $I_1$.
For more examples and details on variable elimination, please refer to \cite{guestrin2001multiagent}.

We have the following result for variable elimination.
\begin{lemma}\label{lemma:variable_elimination_complexity}
    Let $G_1,\ldots, G_\rho$ be the set of agents that belong to group $1,\ldots,\rho$ respectively. Then we have $\localArmSize  = \textstyle{\sum_{e=1}^\rho\prod_{i \in G_e}}|\mathcal{A}_i|$. At every round in \Cref{algo:MATS}, following the above variable elimination procedure, the complexity of searching for the optimal arm is $O(\localArmSize ) = O\left(\textstyle{\sum_{e=1}^\rho\prod_{i \in G_e}|}\mathcal{A}_i|\right)$. 
\end{lemma}

As we discussed, without variable elimination, one would naively add up all the estimated local rewards $\theta_{\bm{a}^e}$ for each joint arm $\bm{a}$ and find the joint arm with the largest posterior $\theta_{\bm{a}}$, leading to computational complexity in the order of $O(A):= O\left(\textstyle{\prod_{i=1}^\rho |\mathcal{A}_i|}\right)$ at each round, which grows exponentially in the number of agents. In contrast, \Cref{lemma:variable_elimination_complexity} indicates that by using variable elimination, \algname\ only needs $\localArmSize $ computation to find the joint arm with the largest estimated reward. Note that this theoretical guarantee is of independent interest to MATS as well since none was given in the original paper \citep{verstraeten2020multiagent}.

\begin{algorithm}[ht]
	\caption{$\epsilon$-Exploring Multi-Agent Thompson Sampling %
 }
	\begin{algorithmic}[1]
		\STATE \textbf{Input}: number of agents $m$, joint arm set $\times_{i=1}^m\mathcal{A}_i$, hyperparameters $c$ and $\epsilon$
		\FOR{$e \in [\rho], \bm{a}^e \in \mathcal{A}^e$}
			\STATE Set $n_{\bm{a}^e}(1)=0$ and $\widehat{\mu}_{\bm{a}^e}(1)= 0$
		\ENDFOR
		\FOR{$t =1,...,T$}
			\FOR{$e \in [\rho], \bm{a}^e \in \mathcal{A}^e$}
				 \STATE %
     {\small
              	\begin{align*}
					\theta_{\bm{a}^e}(t) =
					\begin{cases}
						\sim \mathcal{N}(\widehat{\mu}_{\bm{a}^e}(t), \frac{c}{n_{\ba^e}(t)+1}) & \text{w.p. } \epsilon \\
						=\widehat{\mu}_{\bm{a}^e}(t) & \text{w.p. } 1-\epsilon
					\end{cases}
				\end{align*}}
			\ENDFOR
			\STATE 
   Pick $\bm{A}_t = \argmax_{\bm{a} \in \times_{i=1}^m\mathcal{A}_i} \textstyle{\sum_{e=1}^\rho}\theta_{\bm{a}^e}(t)$ \label{algline:arm_selection}
			\STATE Observe rewards $f^e(\bm{A}_t^e)$ for all $e\in[\rho]$
			\FOR{$e \in [\rho]$}
				\STATE Update {\small $\widehat{\mu}_{\bm{A}_t^e}(t) =\big(n_{\bm{A}_t^e}(t)\widehat{\mu}_{\bm{A}_t^e}(t) + f^e(\bm{A}_t^e)\big)/(n_{\bm{A}_t^e}(t)+1)$}
				\STATE Set $n_{\bm{A}_t^e}(t) = n_{\bm{A}_t^e}(t) + 1$
			\ENDFOR
		\ENDFOR
	\end{algorithmic}
	\label{algo:MATS}
\end{algorithm}
\section{Finite-Time Frequentist Regret Analysis}
In this section, we present the proof of the frequentist regret bound for \algname. 

\subsection{Finite-Time Frequentist Regret Bound of \algname}
For convenience, we use $\Delta_{\bm{a}} = \mu_{\bm{1}} - \mu_{\bm{a}}$ to denote the suboptimality gap between joint arm $\ba$ and the optimal joint arm. We let $\Delta_{\min} = \min_{\bm{a} \in \times_{i=1}^m\mathcal{A}_i\setminus \{\bm{1}\}} \Delta_{\bm{a}}$ and $\Delta_{\max} = \max_{\bm{a} \in \times_{i=1}^m\mathcal{A}_i} \Delta_{\bm{a}}$ be the minimum and maximum gap respectively. Moreover, let   $\Delta_{\bm{a}^e}=\min\{\Delta_{\bm{a}}\mid \bm{a}^e\in \bm{a}\}$ be the minimum reward gap between joint arm $\bm{a}$ which contains $\bm{a}^e$ and $\bm{1}$. We present the regret of \algname\, as follows.
\begin{theorem}\label{thm:upper bound}
 Let $c=\log T.$ The regret of \algname\, satisfies the following results.
 \begin{enumerate}[leftmargin=*,nosep]
     \item There exists some universal constant $C_1$ such that 
     {\small
\begin{align*}
    R_{T}&\leq \frac{C_1(C_1/\epsilon)^{\rho}\rho^2\log^2 (T\localArmSize)}{\Delta_{\min}} \notag \\
    & \quad +C_1\sum_{e\in [\rho]}\sum_{\bm{a}^e\in \cA^{e}\setminus \{\bm{1}^e\}}\frac{\rho^2\log^2 (T\localArmSize)}{\Delta_{\bm{a}^e}}+C_1\Delta_{\max}.
\end{align*}}

\item There exists some universal constant $C_2$ such that
{\small
\begin{align*}
    R_{T} \leq C_2\Delta_{\max}+ C_2\rho \sqrt{\left((C_2/\epsilon)^{\rho}+\localArmSize \right) T\log^2 (T\localArmSize)}.
\end{align*}  }
 \end{enumerate}
\end{theorem}

Note that when $\epsilon =1$, \algname\, reduces to MATS, which gives the first frequentist regret bound of MATS. In particular, our bound is in the same order as the Bayesian regret bound \citep[Theorem 1]{verstraeten2020multiagent} in terms of the order of $T$ and $\localArmSize $. Compared with the Bayesian regret in \citet{verstraeten2020multiagent}, our worst-case regret has an additional $\sqrt{\log T}$ factor because we inflate the variance of posterior distribution by $\log T$, which is a common trick in deriving the worst case regret bound of Thompson sampling \citep{agrawal2017nearoptimal,jin2021mots}. The derivation of $(C_2/\epsilon)^{\rho}$ is provided in the following subsection.

\subsection{Technical Challenges in Frequentist Regret Analysis and the Proof Outline}
\label{sec:chall}
For simplicity, this part assumes $\epsilon=1$ (which reduces to MATS given in \citet{verstraeten2020multiagent}).
First, let's introduce some notations to simplify our discussion. We denote $S_r$ as the set of joint arms with gaps in the interval $(2^{-r},2^{-r+1}]$ and let $\delta_r=2^{-(r+2)}$. The regret incurred by pulling the arms in $S_r$ is represented as $R(S_r)$.
Furthermore, we define $S_{r}(t)$ as the set of joint arms not overestimated at time $t$, formally given by:
\begin{align*}
    S_{r}(t)&=\{\bm{a} \  | \ \bm{a}\in S_{r}, \forall e\in [\rho] \notag \\
    &\quad \text{and} \ \bm{a}^e\neq \bm{1}^e, \theta_{\bm{a}^e}(t)\leq  {\mu}_{\bm{a}^e}+\delta_r/\rho\}.
\end{align*}
The regret $R(S_r)$ can be expressed as
{\small
\begin{align*}
    R(S_r)&\leq 8\delta_r\cdot \big(\underbrace{\textstyle{\sum_{t=1}^{T}}\ind\{\bm{A}_{t}\in S_{r}(t)\}}_{I_{1}}+ \underbrace{\textstyle{\sum_{t=1}^{T}}\ind\{\bm{A}_{t}\notin S_{r}(t)\}}_{I_2} \big).
\end{align*}}
The term $I_2$ is relatively straightforward. It's important to note that whenever we pull the arm $\bm{A}_{t}\notin S_{r}(t)$, we inevitably pull a local arm whose posterior sample has not yet converged to its true mean. After sufficient pulls of each local arm (due to pull $\bm{A}_t$ with $\bm{A}_t\notin S_{r}(t)$), and by employing the maximal inequality for the reward distribution along with the concentration bound for the posterior distribution, we can demonstrate that the event $\bm{A}_{t}\notin S_{r}(t)$ occurs with an exceedingly small probability.

We now discuss how to bound term $I_1$ and the associated challenges. In the regret analysis for single agent Thompson Sampling (TS) by \citet{agrawal2012analysis,jin2022finitetime}, the term $I_1$ is bounded as follows.
{\small
\begin{align}
\label{eq:hard1}
\sum_{t=1}^{T}\ind\{\bm{A}_{t}\in S_{r}(t)\}\leq \EE_{\hat{\mu}_{\bm{1},s}}\bigg[ {1}/{\PP(\theta_{\bm{1},s}\geq \mu_{\bm{1}}-\delta_r)} \bigg],
\end{align}}
where $\hat{\mu}_{\bm{1},s}$ is the empirical mean of arm $\bm{1}$ after being pulled $s$ times, $\theta_{\bm{1},s}$ is the posterior sample from $\cN(\hat{\mu}_{\bm{1},s},c\rho/s)$, and ${1}/{\PP(\theta_{\bm{1},s}\geq \mu_{\bm{1}}-\delta_r)}$ represents the expected maximum number of posterior samples from $\cN(\hat{\mu}_{\bm{1},s},c\rho/s)$ such that one sample is larger than $\mu_{\bm{1}}-\delta_r$.

However, in a multi-agent setting, we can't decompose it in the same way due to two main reasons:
\begin{enumerate}[leftmargin=*,nosep]
    \item Since arm $\bm{1}$ might share some local arms with other joint arms, the number of pulls of each local arm of $\bm{1}$ could be different at time $t$. This contrasts with $\hat{\mu}_{\bm{1},s}$ in Equation \eqref{eq:hard1}, where we assume that arm $1$ is pulled exactly $s$ times.
    \item 
More importantly, while the samples of each local arm are independently drawn from their respective reward distributions, a dependency issue arises in the case of the joint arm. To explain, if each local arm $\bm{1}^e$ is pulled a fixed number of times, $N_{e}$, the mean reward of $\bm{1}$ follows the $\left(\sqrt{\textstyle{\sum_{e\in [\rho]}}(N_{e})^{-1}}\right)$-subgaussian with a mean of $\textstyle{\sum_{e\in [\rho]}}\hat{\mu}_{\bm{1}^e, N{e}}$. Leveraging the properties of subgaussian random variables, it can be demonstrated that the mean reward of joint arm $\bm{1}$ converges to its true mean as the number of pulls increases. However, this is not true when the pulls of local arms are history-dependent. In such cases, MATS is more likely to pull suboptimal arms that share overestimated local arms of $\bm{1}$ (the posterior samples from these local arms could surpass those from other underestimated local arms of $\bm{1}$). If this situation occurs,   $\hat{\mu}_{\bm{1}}(t)$ is likely to be underestimated, making its distribution challenging to ascertain. Therefore, it will be difficult to derive the concentration results for $\hat{\mu}_{\bm{1}}(t)$ and consequently, the probability of $\theta_{\bm{1}}(t)\geq \mu_{\bm{1}}-\delta_r$ would also be hard to establish.  
\end{enumerate}
We solve the above issues by 1: carefully dividing $S_r$ into subsets, where the arms in each subset share the same local arms with $\bm{1}$ (total $2^{\rho}$ subsets); and 2: bounding term $I_1$ in local arms level. These two methods allow us to prove that 
{\small
\begin{align}
\label{eq:hard2}
     &\sum_{t=1}^{T}\ind\{\bm{A}_{t}\in S_{r}(t)\} \notag \\
\leq & 2^{\rho} \sum_{s=1}^{\tau}\prod_{e=1}^{\rho} \EE_{\hat{\mu}_{\bm{1}^e,s}}\left[ 1/{\PP(\theta_{\bm{1}^{e},s}\geq \mu_{\bm{1}^e}-\delta_r/{\rho} )}\right]+\Theta(1).
\end{align}}
In the right hand of inequality, $2^{\rho}$ is due to the existence of $2^{\rho}$ subsets, and $\tau$ is defined as $\Theta({\rho^2(\log (T\localArmSize))^{2}}/{(\delta_r)^2})$. The term $\Theta(1)$ exists because after each local arm is pulled more than $\tau$ times, the event $\PP\left(\theta_{\bm{1}^{e},s}\geq \mu_{\bm{1}^e}-{\delta_r}/{\rho} \right)$ is highly likely to occur. The cost for the non-occurrence of this event can be bounded by $\Theta(1)$.
Follow \citet{agrawal2012analysis,jin2022finitetime}, one can show that
{\small
\begin{align*}
\prod_{e=1}^{\rho} \EE_{\hat{\mu}_{\bm{1}^e,s}}\left[ {1}/{\PP\left(\theta_{\bm{1}^{e},s}\geq \mu_{\bm{1}^e}-\frac{\delta_r}{\rho} \right)}\right]\leq \big(\frac{\rho\sqrt{\log T}}{\delta_r} \big)^{2\rho}.
\end{align*}}
The above results are underwhelming, particularly in regards to the worst-case regret, which is $\tilde O(T^{2\rho/(2\rho-1)})$. In order to enhance these outcomes, we are introducing two innovative techniques:

\begin{enumerate}[leftmargin=*,nosep]
\item First, deriving from the concentration bound, we obtain that with high probability  $\hat{\mu}_{\bm{1}^{e},s}\geq \mu_{\bm{1}^e}-\sqrt{2\log T/s}$. Instead of considering the expectation over the entire real line for $\hat{\mu}_{\bm{1}^e,s}$, we confine $\hat{\mu}_{\bm{1}^e,s}$ to the interval $(\mu_{\bm{1}^e}-\sqrt{2\log T/s}, \infty)$. 
\item Secondly, we marginally increase the variance of the posterior distribution by $\log T$. According to the anti-concentration bound of the Gaussian posterior, the likelihood of $\theta_{\bm{1}^{e},s}$ exceeding $\hat{\mu}_{\bm{1}^e,s}+\sqrt{2\log T/s}$ remains constant. In conjunction with the condition $\hat{\mu}_{\bm{1}^{e},s}\geq \mu_{\bm{1}^e}-\sqrt{2\log T/s}$, we can ascertain that $\PP(\theta_{\bm{1}^{e},s}\geq \mu_{\bm{1}^e})$ is also a constant.
\end{enumerate}
With the above two methods, we can prove 
$I_1\leq \tau\cdot C^{\rho}$,
where $C$ is some constant.  Finally, by summing over all possible values of $r$ (i.e., $\sum_{r} R(S_r)$), we derive our regret bound.

\subsection{Lower Bound on the Worst-Case Regret Bound}
We now present some lower-bound results on the worst-case regret in our setting. The first theorem states a lower bound in terms of the horizon length $T$ and the total number of local arms across all groups $\localArmSize$ when $\rho>0$ is treated as a fixed constant.
\begin{theorem}\label{thm:lower_bound}
    For any policy $\pi$, there exists a mean vector $\mu \in [0, 1]^{\localArmSize }$ (where each component corresponds to the mean of a local arm) such that $R_n(\pi, \nu_\mu) = \Omega(\sqrt{\localArmSize  T/\rho} )$.
\end{theorem}    
Recall that our worst-case regret bound in  \Cref{thm:upper bound} is $\tilde{O}(\sqrt{C^{\rho}\localArmSize T})$, with $C$ representing a universal constant. According to Theorem \ref{thm:lower_bound}, when the number of groups $\rho$ is constant, indicating a sparse hypergraph, our worst-case regret for \algname\ is nearly optimal up to constant and logarithmic terms. 

The next theorem shows the worst possible dependence of the regret bound of \algname\ on the number of groups, i.e., $\rho$.
\begin{theorem}\label{thm:lcb}
For $c=\epsilon=1$, there is a bandit instance such that the regret of Algorithm \ref{algo:MATS} is $\Omega(C^{\rho})$, where $C>1$ is some constant.
\end{theorem}
\Cref{thm:lcb} shows that $C^{\rho}$ regret is unavoidable for original Multi-Agent Thompson Sampling, which further proves the optimality of our regret bound $\tilde{O}(\sqrt{C^{\rho}\localArmSize T})$.

\begin{figure}[h]
    \centering
    \includegraphics[width = .45\textwidth]{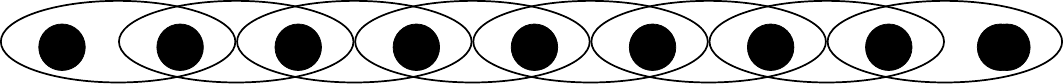}
    \caption{The hypergraph on a 0101-Chain with $10$ agents. Each hyperedge (a group of two agents) is denoted by a black oval.
    }\label{fig:chain_graph}
\end{figure}
\begin{figure*}
        \centering
        \subfigure[Bernoulli 0101]{
            \includegraphics[height=0.18\textwidth]{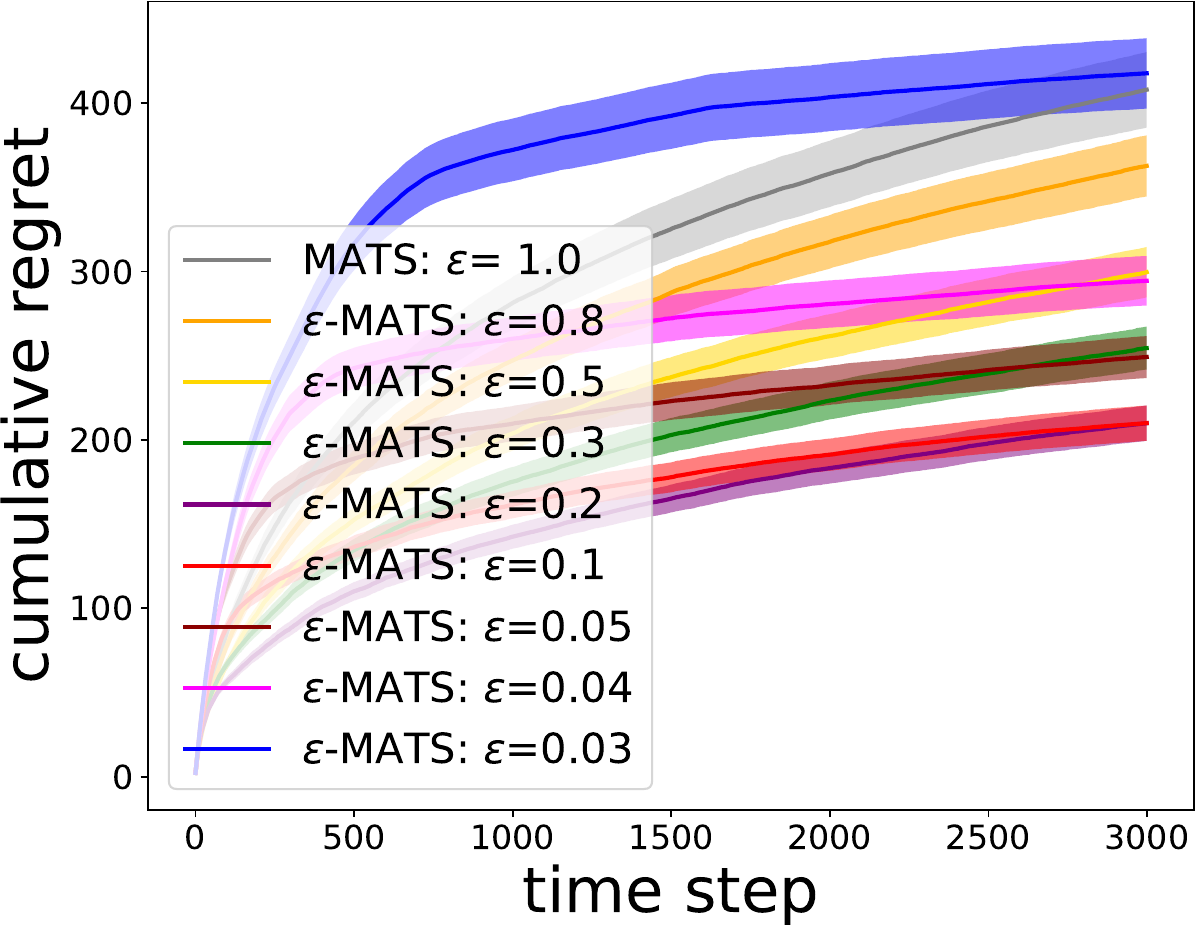}\label{fig:bernoulli_2agentInGroup_10agent_epsilon}
        }
        \subfigure[Poisson0101]{
            \includegraphics[height=0.18\textwidth]{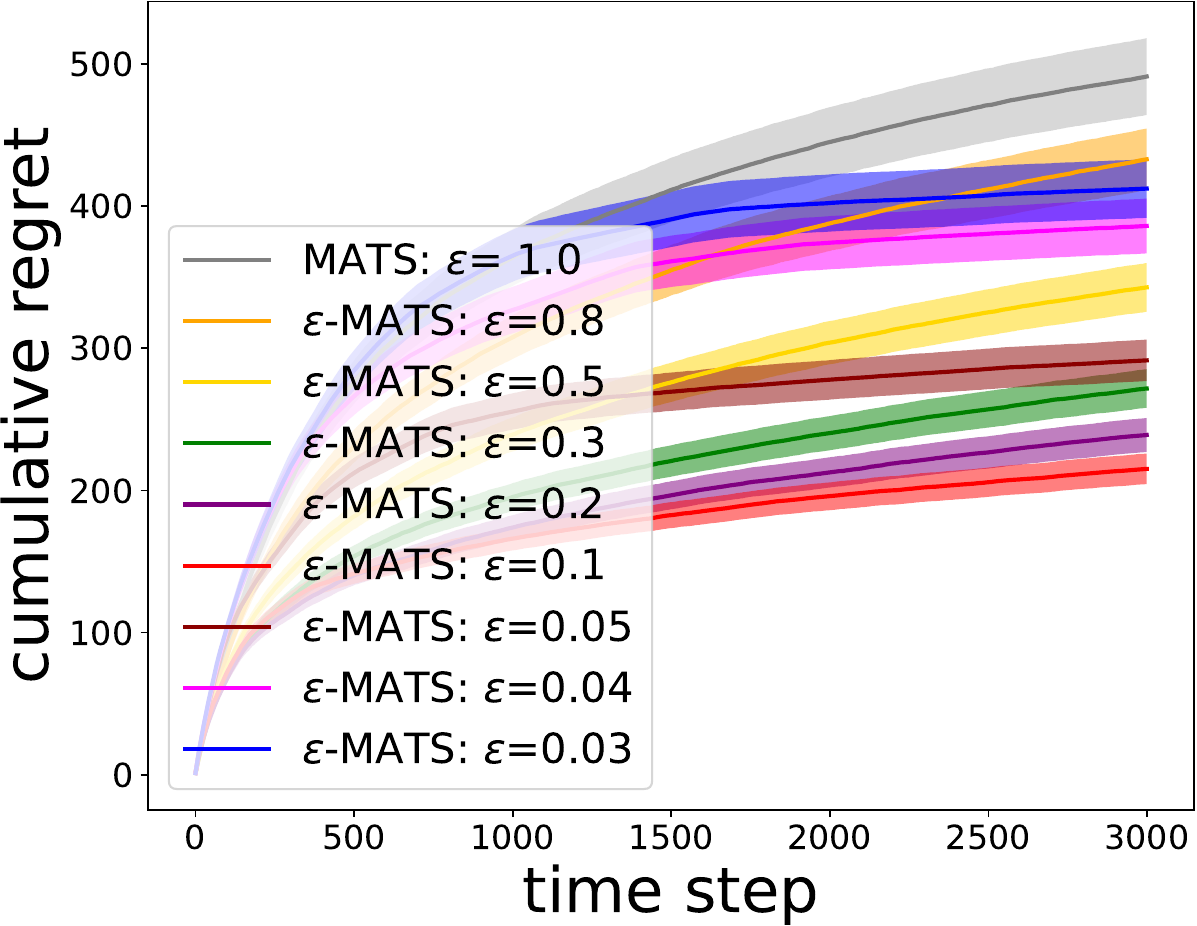}\label{fig:poisson_2agentInGroup_10agent_epsilon}
        }
        \subfigure[Bernoulli 0101 and Poisson0101 with $d = 2$]{
            \includegraphics[height=0.18\textwidth]{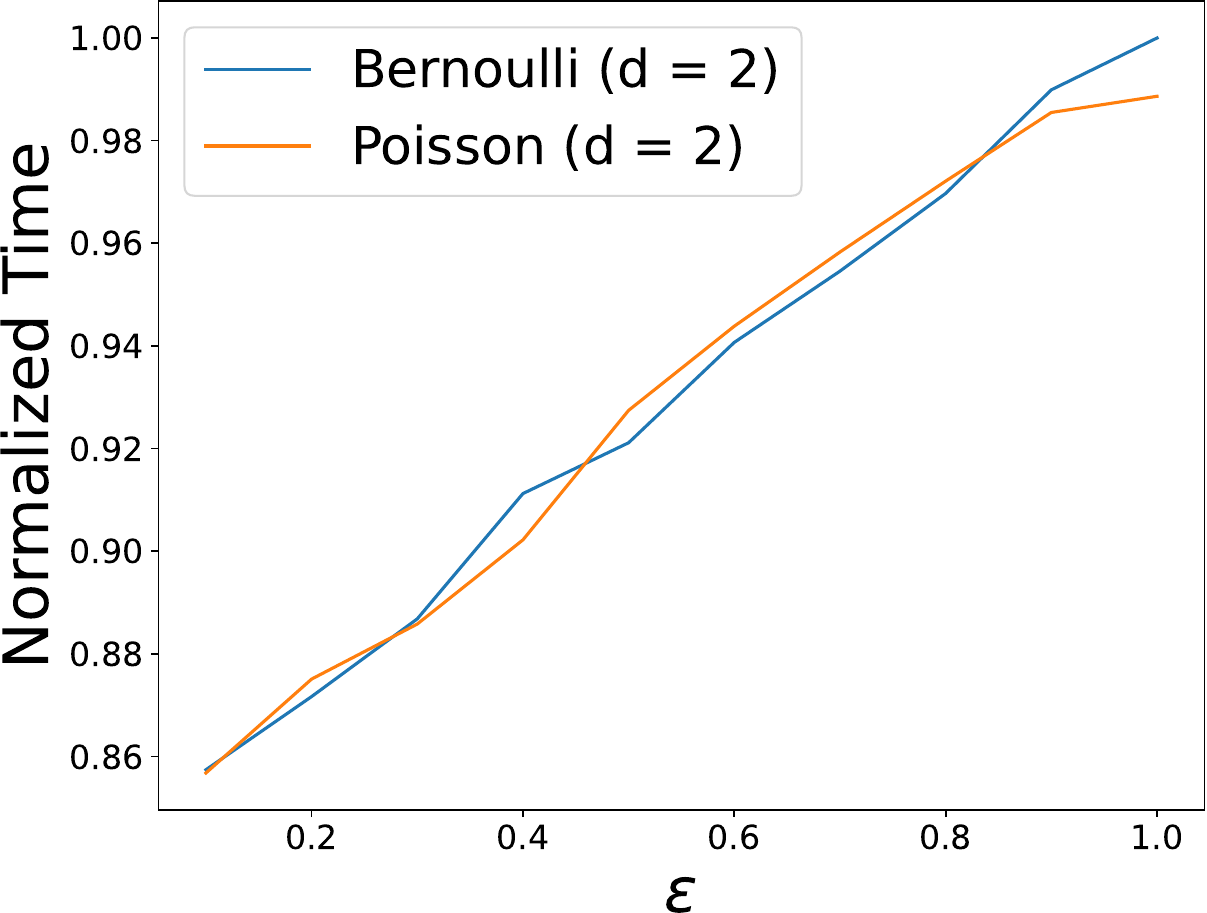}\label{fig:hypergraph_2inGroup_computation}
        }
        \subfigure[Bernoulli 0101 and Poisson0101 with $d = 3$]{
            \includegraphics[height=0.18\textwidth]{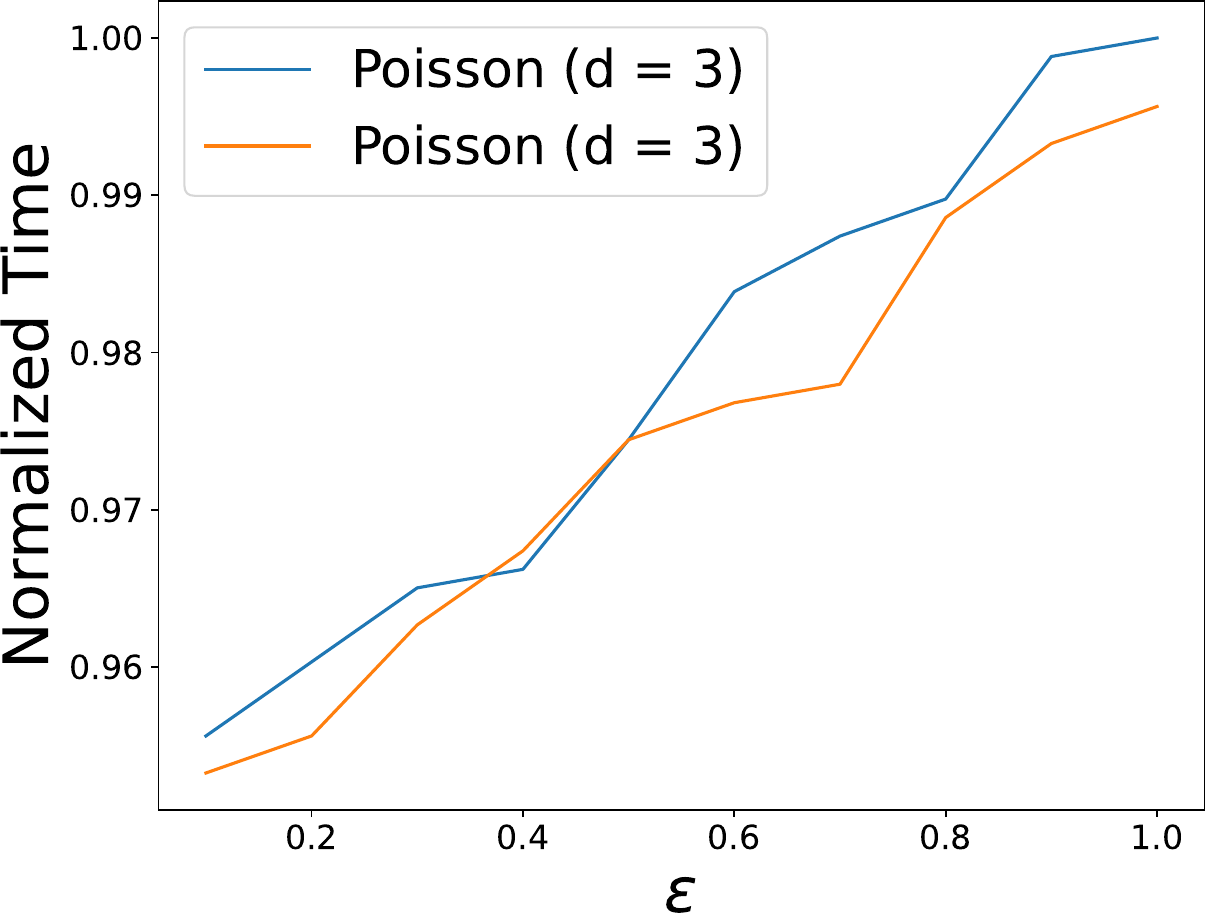}\label{fig:hypergraph_3inGroup_computation}
        }   
        \caption {Ablation study on \algname. (a) and (b): Regret performance ($m = 10$, and $d = 2$) with different $\epsilon$ in Bernoulli 0101 and Poisson 0101 tasks. Note when $\epsilon = 1.0$, \algname\, reduces to MATS.  (c) and (d): The relative computational time of \algname\, with different $\epsilon$ compared with MATS. }
        \label{fig:ablation_epsilon_in_main_paper}
\end{figure*}
\begin{figure*}[ht]
     \centering
     \subfigure[Bernoulli 0101: $m = 10$, $d = 2$]{
         \includegraphics[height=0.176\linewidth]{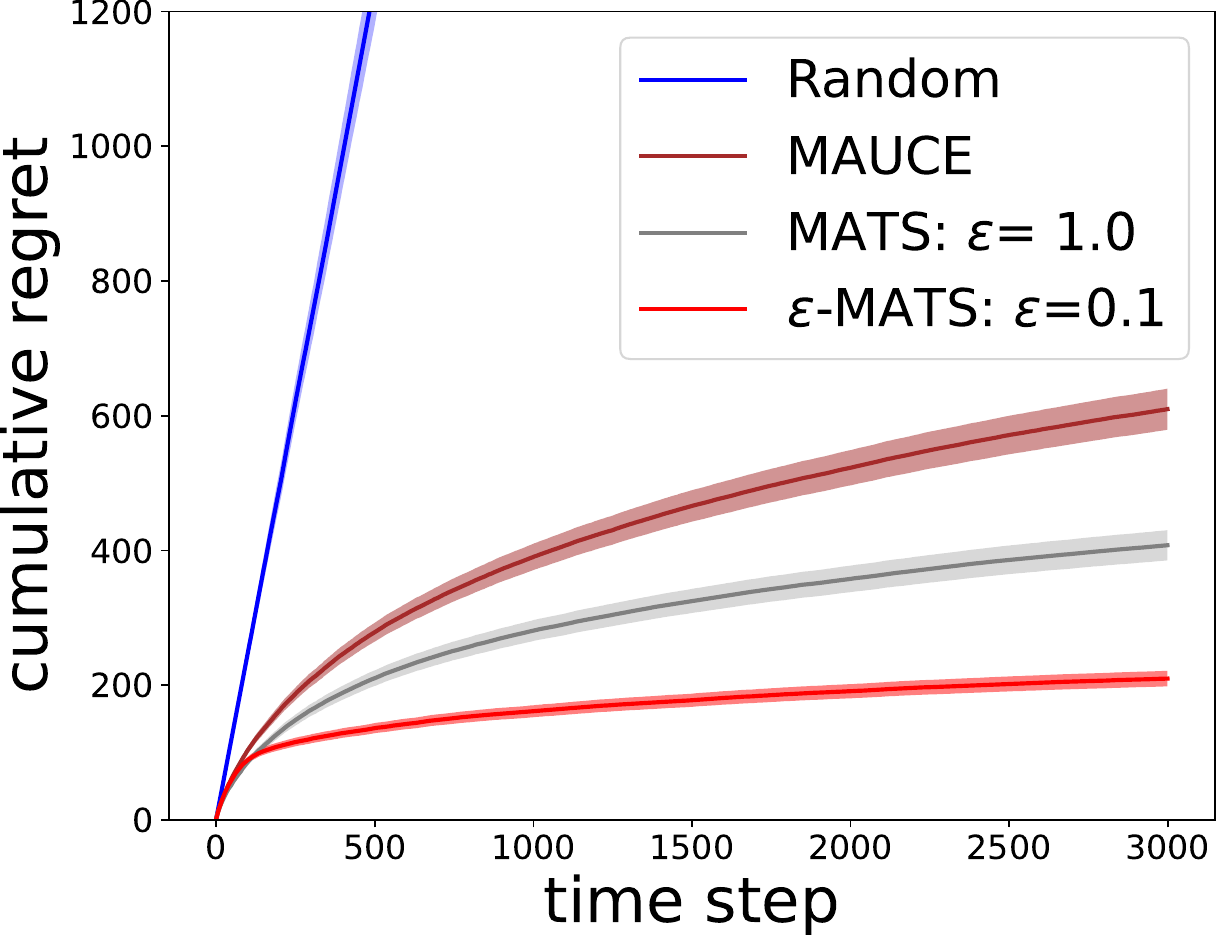}
         \label{fig:bernoulli_2agentInGroup_10agent_main}
     }
     \subfigure[Poisson 0101: $m = 10$, $d = 2$]{       
         \includegraphics[height=0.176\linewidth]{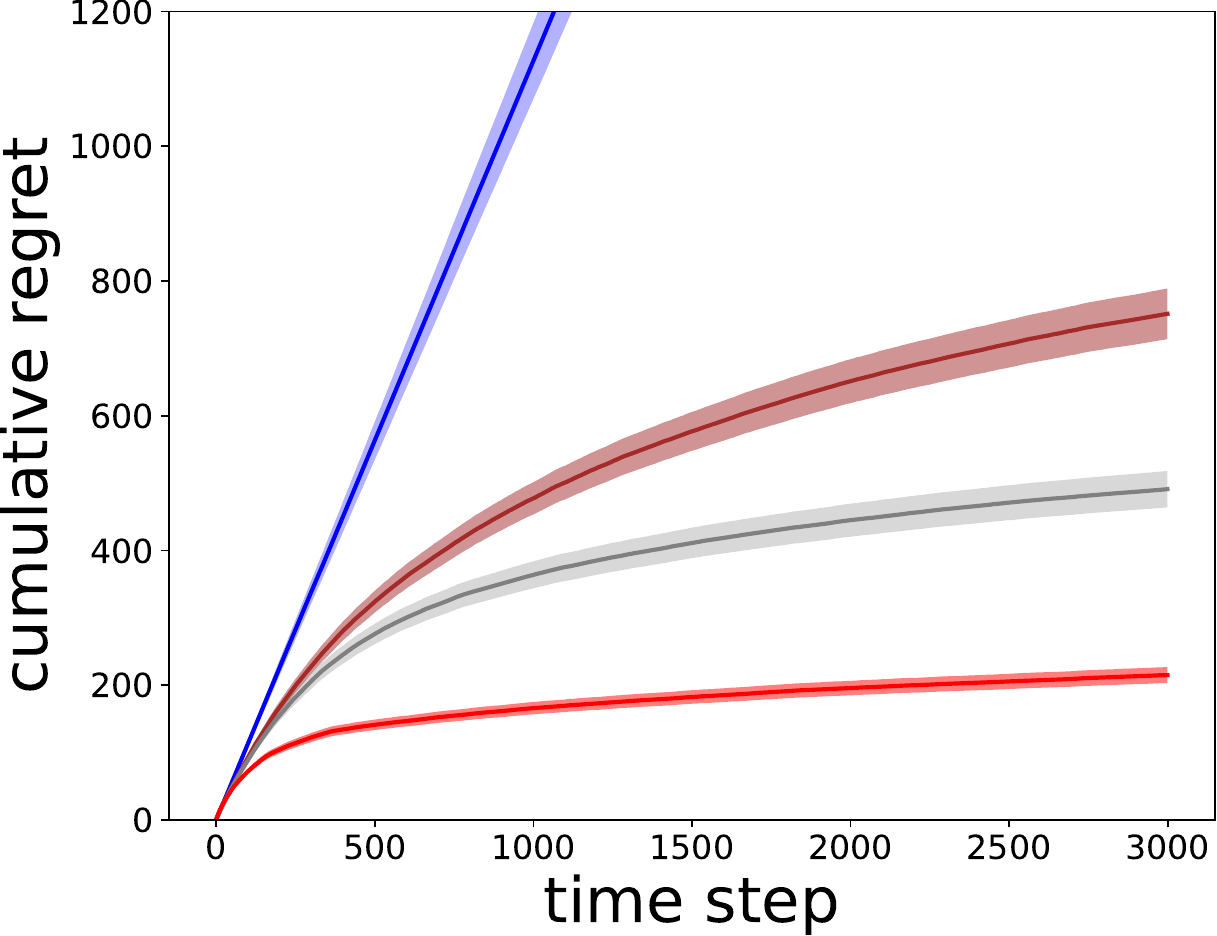}
         \label{fig:poisson_2agentInGroup_10agent_main}
     }
     \subfigure[Bernoulli 0101: $m = 10$, $d = 3$]{       
         \includegraphics[height=0.176\linewidth]{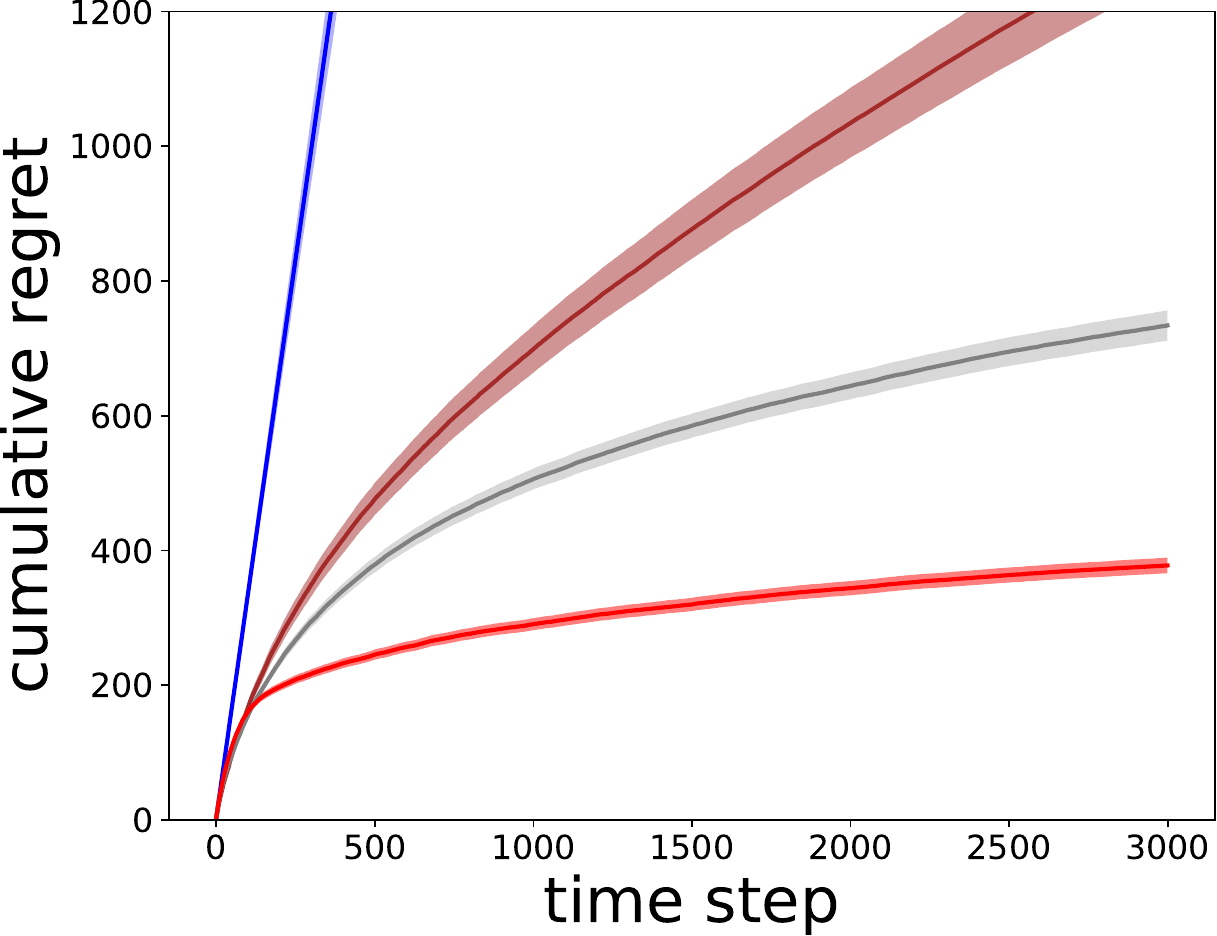}
         \label{fig:bernoulli_3agentInGroup_10agent_main}
     }
      \subfigure[Poisson 0101: $m = 10$, $d = 3$]{       
         \includegraphics[height=0.176\linewidth]{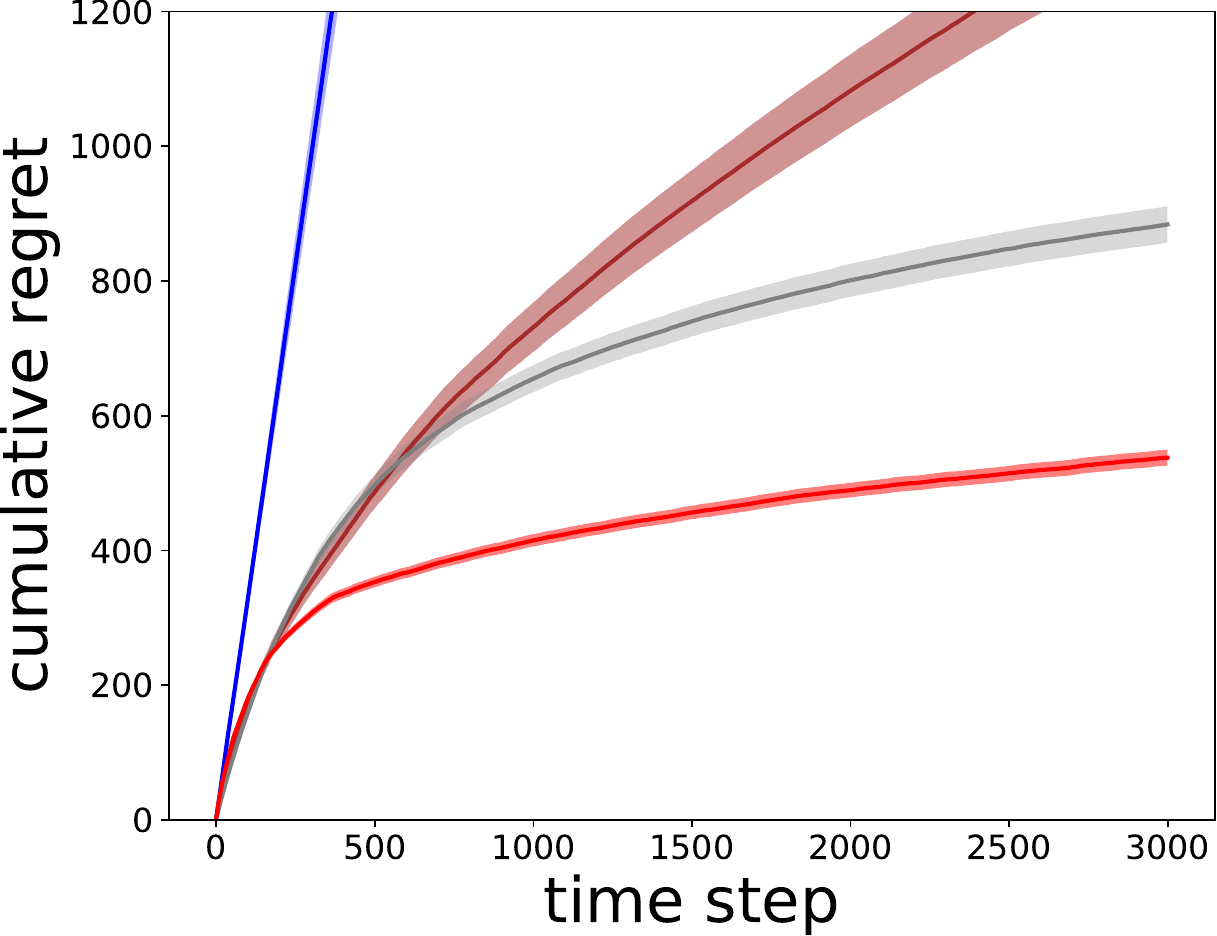}
         \label{fig:poisson_3agentInGroup_10agent_main}
     }
      \caption{Regret performance compared with other algorithm baselines in Bernoulli 0101 and Poisson 0101 with different agents in a group ($d = 2$ or $d = 3)$. \label{fig:baseline_in_main_paper}}
\end{figure*}

\section{Experiments}\label{sec:experiment}
In this section, we evaluate the proposed \algname\, algorithm on benchmark MAMAB problems including the Bernoulli 0101-Chain and the Poisson 0101-Chain. We also conducted experiments on the Gem Mining problem and the Wind Farm Control problem, which can be found in \Cref{sec:exp_details} \citep{roijers2015computing, bargiacchi2018learning, verstraeten2020multiagent}. We compare \algname\, with the vanilla MATS \citep{verstraeten2020multiagent}, MAUCE \citep{bargiacchi2018learning}, and the random policy. We also provide a thorough ablation study of \algname\, to find the optimal trade-off between greedy and Thompson sampling exploration in practice. We run all our experiments on Nvidia RTX A5000 with 24GB RAM and each experiment setting is averaged over 50 trials.  Please refer to \Cref{sec:exp_details} for detailed experimental setup, ablation studies, and more experimental results. The implementation of our \algname\ algorithm is available at \url{https://github.com/panxulab/eps-Multi-Agent-Thompson-Sampling}.

\subsection{Bernoulli and Poisson 0101-Chain}\label{sec:exp_hypergraph}
In this subsection, we conduct experiments on the Bernoulli and Poisson 0101-Chain problems, which are commonly studied in the MAMAB literature \citep{bargiacchi2018learning, verstraeten2020multiagent}. An illustration is provided in
Figure \ref{fig:chain_graph}, where the agents are positioned along a 1-dimensional path forming a graph. 
Specifically, the graph consists of $m$ agents and $m - (d -1)$ edges (or groups), where $d$ is the number of agents within each hyperedge. The agents $i$ to $i + d - 1$ in the group $i$ are connected to a local reward function $f^i (a_i, a_{i+1}, ..., a_{i+d-1})$, where $a_i$ denotes the individual arm of agent $i$.  Each agent can has two arms: $0$ and $1$. 

We consider two settings where $d = 2$ and $d=3$ respectively. For each setting, we conduct experiments for two types of reward distributions (Bernoulli and Poisson), which results in the \textbf{Bernoulli 0101-Chain} problem and the
\textbf{Poisson 0101-Chain} problem respectively. Due to the space limit, we defer the details of the reward generation to \Cref{sec:exp_details}.

We first perform an ablation study to show the effect of different $\epsilon$ on the performance of \algname. The results are presented in Figures \ref{fig:bernoulli_2agentInGroup_10agent_epsilon} and \ref{fig:poisson_2agentInGroup_10agent_epsilon}. It can be seen that with $\epsilon$ decreasing from $1$ (this corresponds to the MATS algorithm) to $0.1$, the cumulative regret of \algname\, also becomes lower. When $\epsilon$ gets smaller than $0.1$, the regret rapidly increases due to insufficient exploration. In the rest of the experiments in this subsection, we fix $\epsilon = 0.1$ for the best performance. We also compare the runtime complexity of \algname\, for different $\epsilon$, which is presented in Figures \ref{fig:hypergraph_2inGroup_computation} and \ref{fig:hypergraph_3inGroup_computation} for the setting $d=2$ and $d=3$ respectively. In particular, we calculate the ratio between the runtime of \algname\, and MATS ($\epsilon=1$) for running 1000 iterations.   Figure \ref{fig:hypergraph_2inGroup_computation} shows that lower value of $\epsilon$ decreases the runtime complexity of \algname\ in both Bernoulli 0101 and Poisson 0101 problems. Comparing  \Cref{fig:hypergraph_2inGroup_computation} and \Cref{fig:hypergraph_3inGroup_computation}, we can see that the computational efficiency is adversely affected by the size of each group. In Figure \ref{fig:baseline_in_main_paper}, we compare the regret of \algname\, with MATS, MAUCE, and Random for both Bernoulli 0101 and Poisson 0101 tasks, which demonstrate that \algname\, can significantly outperform baseline methods.

\section{Conclusion and Future Work}
In this paper, we studied the problem of multi-agent multi-armed bandits. We proposed \algname\, which combines the MATS exploration with probability $\epsilon$ and greedy exploitation with probability $1-\epsilon$. We provided a frequentist finite time regret bound for \algname, which is in the order of $\tilde{O}(\sqrt{C^\rho\localArmSize  T})$. When $\epsilon=1$, our result yields the first frequentist regret bound for MATS in the coordination hypergraph setting \cite{verstraeten2020multiagent}. We also derived a lower bound for this environment in the order of $\Omega(\sqrt{\localArmSize  T/\rho})$, implying \algname\, is near optimal when $\rho$ is assumed to be small, i.e., the coordination hypergraph is sparse. Empirical evaluations demonstrate the superior performance of \algname\, compared with existing algorithms for MAMAB problems with a coordination hypergraph. 

Our experimental findings present a notable observation: the performance of \algname\ frequently surpasses that of MATS. Nevertheless, the regret bound presented in our main theorem suggests that MATS has a slightly better worst-case regret bound compared to \algname. This discrepancy offers a compelling avenue for future research to explore the potential for \algname\ to achieve the same regret bound as MATS. Additionally, while this paper focuses on the worst-case regret bound, it would be intriguing to investigate if \algname\ could exhibit a more favorable regret bound in some easier bandit instances, leading to the analysis of instance-dependent regret bounds. Finally, it would be intriguing to investigate the potential of applying the core concepts of coordination hypergraph and epsilon-exploring in our paper to enhance Thompson sampling-based algorithms within more complex settings, such as linear bandits, neural contextual bandits, and Markov decision processes \citep{xu2022langevin,xu2022neural,zhang2020neural,ishfaq2023provable}.

\noindent\textbf{Acknowledgements} We thank the anonymous reviewers for their helpful comments. This research is supported by the National Research Foundation, Singapore under its AI Singapore Program (AISG Award No: AISG-PhD/2021-01-004[T]), 
by the Whitehead Scholars Program, and by the US National Science Foundation Award 2323112. In particular, T. Jin is supported by the National Research Foundation, Singapore under its AI Singapore Program (AISG Award No: AISG-PhD/2021-01-004[T]). P. Xu is supported by the Whitehead Scholars Program and the
US National Science Foundation Award 2323112.     
The views in this paper are those of the authors and do not represent any funding agencies.
\bibliography{mats_sorted}

\begin{thebibliography}{61}
\providecommand{\natexlab}[1]{#1}

\bibitem[{Abramowitz and Stegun(1964)}]{abramowitz1964handbook}
Abramowitz, M.; and Stegun, I.~A. 1964.
\newblock \emph{Handbook of Mathematical Functions with Formulas, Graphs, and Mathematical Tables}, volume~55.
\newblock {US Government printing office}.

\bibitem[{Agarwal, Aggarwal, and Azizzadenesheli(2022)}]{agarwal2022multiagent}
Agarwal, M.; Aggarwal, V.; and Azizzadenesheli, K. 2022.
\newblock Multi-Agent Multi-Armed Bandits with Limited Communication.
\newblock \emph{Journal of Machine Learning Research}, 23(212): 1--24.

\bibitem[{Agrawal and Goyal(2012)}]{agrawal2012analysis}
Agrawal, S.; and Goyal, N. 2012.
\newblock Analysis of Thompson Sampling for the Multi-Armed Bandit Problem.
\newblock In \emph{Conference on Learning Theory}, 39--1. {JMLR Workshop and Conference Proceedings}.

\bibitem[{Agrawal and Goyal(2017)}]{agrawal2017nearoptimal}
Agrawal, S.; and Goyal, N. 2017.
\newblock Near-Optimal Regret Bounds for Thompson Sampling.
\newblock \emph{Journal of the ACM (JACM)}, 64(5): 1--24.

\bibitem[{Auer et~al.(2002)Auer, {Cesa-Bianchi}, Freund, and Schapire}]{auer2002nonstochastic}
Auer, P.; {Cesa-Bianchi}, N.; Freund, Y.; and Schapire, R.~E. 2002.
\newblock The Nonstochastic Multiarmed Bandit Problem.
\newblock \emph{SIAM journal on computing}, 32(1): 48--77.

\bibitem[{Bargiacchi et~al.(2018)Bargiacchi, Verstraeten, Roijers, Now{\'e}, and Hasselt}]{bargiacchi2018learning}
Bargiacchi, E.; Verstraeten, T.; Roijers, D.; Now{\'e}, A.; and Hasselt, H. 2018.
\newblock Learning to Coordinate with Coordination Graphs in Repeated Single-Stage Multi-Agent Decision Problems.
\newblock In \emph{International Conference on Machine Learning}, 482--490. {PMLR}.

\bibitem[{Besson and Kaufmann(2018)}]{besson2018multiplayer}
Besson, L.; and Kaufmann, E. 2018.
\newblock Multi-Player Bandits Revisited.
\newblock In \emph{Algorithmic Learning Theory}, 56--92. {PMLR}.

\bibitem[{Boursier and Perchet(2020)}]{boursier2020selfish}
Boursier, E.; and Perchet, V. 2020.
\newblock Selfish Robustness and Equilibria in Multi-Player Bandits.
\newblock In \emph{Conference on Learning Theory}, 530--581. {PMLR}.

\bibitem[{Chang, {Jafarnia-Jahromi}, and Jain(2022)}]{chang2022online}
Chang, W.; {Jafarnia-Jahromi}, M.; and Jain, R. 2022.
\newblock Online Learning for Cooperative Multi-Player Multi-Armed Bandits.
\newblock In \emph{{{IEEE}} Conference on Decision and Control ({{CDC}})}, 7248--7253. {IEEE}.

\bibitem[{Chapelle and Li(2011)}]{chapelle2011empirical}
Chapelle, O.; and Li, L. 2011.
\newblock An Empirical Evaluation of {{Thompson}} Sampling.
\newblock In \emph{Advances in Neural Information Processing Systems}, 2249--2257.

\bibitem[{Claes et~al.(2017)Claes, Oliehoek, Baier, and Tuyls}]{claes2017decentralised}
Claes, D.; Oliehoek, F.; Baier, H.; and Tuyls, K. 2017.
\newblock Decentralised Online Planning for Multi-Robot Warehouse Commissioning.
\newblock In \emph{Conference on Autonomous Agents and {{MultiAgent}} Systems}, 492--500.

\bibitem[{De~Hauwere, Vrancx, and Now{\'e}(2010)}]{dehauwere2010learning}
De~Hauwere, Y.-M.; Vrancx, P.; and Now{\'e}, A. 2010.
\newblock Learning Multi-Agent State Space Representations.
\newblock In \emph{International Conference on Autonomous Agents and Multiagent Systems: Volume 1-{{Volume}} 1}, 715--722.

\bibitem[{Deng et~al.(2023)Deng, Zhang, Xu, Ma, and Gu}]{deng2023phygcn}
Deng, Y.; Zhang, R.; Xu, P.; Ma, J.; and Gu, Q. 2023.
\newblock PhyGCN: Pre-trained Hypergraph Convolutional Neural Networks with Self-supervised Learning.
\newblock \emph{bioRxiv}, 2023--10.

\bibitem[{Elahi et~al.(2021)Elahi, Atalar, {\"O}{\u g}{\"u}t, and Tekin}]{elahi2021contextual}
Elahi, S.; Atalar, B.; {\"O}{\u g}{\"u}t, S.; and Tekin, C. 2021.
\newblock Contextual Combinatorial Volatile Bandits with Satisfying via Gaussian Processes.
\newblock \emph{arXiv preprint arXiv:2111.14778}.

\bibitem[{Gebraad and {van Wingerden}(2015)}]{gebraad2015maximum}
Gebraad, P.~M.; and {van Wingerden}, J.-W. 2015.
\newblock Maximum Power-Point Tracking Control for Wind Farms.
\newblock \emph{Wind Energy}, 18(3): 429--447.

\bibitem[{Gentile, Li, and Zappella(2014)}]{gentile2014online}
Gentile, C.; Li, S.; and Zappella, G. 2014.
\newblock Online Clustering of Bandits.
\newblock In \emph{International Conference on Machine Learning}, 757--765. {PMLR}.

\bibitem[{Guestrin, Koller, and Parr(2001)}]{guestrin2001multiagent}
Guestrin, C.; Koller, D.; and Parr, R. 2001.
\newblock Multiagent Planning with Factored {{MDPs}}.
\newblock \emph{Advances in neural information processing systems}, 14.

\bibitem[{Gupta et~al.(2021)Gupta, Chaudhari, Joshi, and Ya{\u g}an}]{gupta2021multiarmed}
Gupta, S.; Chaudhari, S.; Joshi, G.; and Ya{\u g}an, O. 2021.
\newblock Multi-Armed Bandits with Correlated Arms.
\newblock \emph{IEEE Transactions on Information Theory}, 67(10): 6711--6732.

\bibitem[{Han and Arndt(2021)}]{han2021budget}
Han, B.; and Arndt, C. 2021.
\newblock Budget Allocation as a Multi-Agent System of Contextual \& Continuous Bandits.
\newblock In \emph{{{ACM SIGKDD}} Conference on Knowledge Discovery \& Data Mining}, 2937--2945.

\bibitem[{Hayes et~al.(2022)Hayes, Verstraeten, Roijers, Howley, and Mannion}]{hayes2022multiobjective}
Hayes, C.~F.; Verstraeten, T.; Roijers, D.~M.; Howley, E.; and Mannion, P. 2022.
\newblock Multi-Objective Coordination Graphs for the Expected Scalarised Returns with Generative Flow Models.
\newblock \emph{arXiv preprint arXiv:2207.00368}.

\bibitem[{Huang, Combes, and Trinh(2022)}]{huang2022optimal}
Huang, W.; Combes, R.; and Trinh, C. 2022.
\newblock Towards Optimal Algorithms for Multi-Player Bandits without Collision Sensing Information.
\newblock In \emph{Conference on Learning Theory}, 1990--2012. {PMLR}.

\bibitem[{Ishfaq et~al.(2023)Ishfaq, Lan, Xu, Mahmood, Precup, Anandkumar, and Azizzadenesheli}]{ishfaq2023provable}
Ishfaq, H.; Lan, Q.; Xu, P.; Mahmood, A.~R.; Precup, D.; Anandkumar, A.; and Azizzadenesheli, K. 2023.
\newblock Provable and Practical: Efficient Exploration in Reinforcement Learning via Langevin Monte Carlo.
\newblock \emph{arXiv preprint arXiv:2305.18246}.

\bibitem[{Jin et~al.(2021{\natexlab{a}})Jin, Xu, Shi, Xiao, and Gu}]{jin2021mots}
Jin, T.; Xu, P.; Shi, J.; Xiao, X.; and Gu, Q. 2021{\natexlab{a}}.
\newblock MOTS: {{Minimax}} Optimal Thompson Sampling.
\newblock In \emph{International Conference on Machine Learning}, 5074--5083. {PMLR}.

\bibitem[{Jin et~al.(2022)Jin, Xu, Xiao, and Anandkumar}]{jin2022finitetime}
Jin, T.; Xu, P.; Xiao, X.; and Anandkumar, A. 2022.
\newblock Finite-Time Regret of Thompson Sampling Algorithms for Exponential Family Multi-Armed Bandits.
\newblock \emph{Advances in Neural Information Processing Systems}, 35: 38475--38487.

\bibitem[{Jin et~al.(2021{\natexlab{b}})Jin, Xu, Xiao, and Gu}]{jin2021double}
Jin, T.; Xu, P.; Xiao, X.; and Gu, Q. 2021{\natexlab{b}}.
\newblock Double Explore-Then-Commit: {{Asymptotic}} Optimality and Beyond.
\newblock In \emph{Conference on Learning Theory}, 2584--2633. {PMLR}.

\bibitem[{Jin et~al.(2023)Jin, Yang, Xiao, and Xu}]{jin2023thompson}
Jin, T.; Yang, X.; Xiao, X.; and Xu, P. 2023.
\newblock Thompson Sampling with Less Exploration is Fast and Optimal.
\newblock In \emph{International Conference on Machine Learning}, 15239--15261. PMLR.

\bibitem[{Kaufmann, Korda, and Munos(2012)}]{kaufmann2012thompson}
Kaufmann, E.; Korda, N.; and Munos, R. 2012.
\newblock Thompson Sampling: {{An}} Asymptotically Optimal Finite-Time Analysis.
\newblock In \emph{Algorithmic Learning Theory}, 199--213. {Springer}.

\bibitem[{Koc{\'a}k and Garivier(2021)}]{kocak2021best}
Koc{\'a}k, T.; and Garivier, A. 2021.
\newblock Best Arm Identification in Spectral Bandits.
\newblock In \emph{Proceedings of the Twenty-Ninth International Conference on International Joint Conferences on Artificial Intelligence}, 2220--2226.

\bibitem[{Kok and Vlassis(2006)}]{kok2006collaborative}
Kok, J.~R.; and Vlassis, N. 2006.
\newblock Collaborative Multiagent Reinforcement Learning by Payoff Propagation.
\newblock \emph{Journal of Machine Learning Research}, 7: 1789--1828.

\bibitem[{Korda, Kaufmann, and Munos(2013)}]{korda2013thompson}
Korda, N.; Kaufmann, E.; and Munos, R. 2013.
\newblock Thompson Sampling for 1-Dimensional Exponential Family Bandits.
\newblock \emph{Advances in neural information processing systems}, 26.

\bibitem[{Landgren, Srivastava, and Leonard(2016)}]{landgren2016distributed}
Landgren, P.; Srivastava, V.; and Leonard, N.~E. 2016.
\newblock Distributed Cooperative Decision-Making in Multiarmed Bandits: {{Frequentist}} and Bayesian Algorithms.
\newblock In \emph{{{IEEE}} Conference on Decision and Control ({{CDC}})}, 167--172. {IEEE}.

\bibitem[{Li, Karatzoglou, and Gentile(2016)}]{li2016collaborative}
Li, S.; Karatzoglou, A.; and Gentile, C. 2016.
\newblock Collaborative Filtering Bandits.
\newblock In \emph{{{ACM SIGIR}} Conference on Research and Development in Information Retrieval}, 539--548.

\bibitem[{Lugosi and Mehrabian(2022)}]{lugosi2022multiplayer}
Lugosi, G.; and Mehrabian, A. 2022.
\newblock Multiplayer Bandits without Observing Collision Information.
\newblock \emph{Mathematics of Operations Research}, 47(2): 1247--1265.

\bibitem[{Ma, Huang, and Schneider(2015)}]{ma2015active}
Ma, Y.; Huang, T.-K.; and Schneider, J.~G. 2015.
\newblock Active Search and Bandits on Graphs Using Sigma-Optimality.
\newblock In \emph{Uncertainty in Artificial Intelligence}, volume 542, 551.

\bibitem[{Magesh and Veeravalli(2019)}]{magesh2019multiuser}
Magesh, A.; and Veeravalli, V.~V. 2019.
\newblock Multi-User {{MABs}} with User Dependent Rewards for Uncoordinated Spectrum Access.
\newblock In \emph{Asilomar Conference on Signals, Systems, and Computers}, 969--972. {IEEE}.

\bibitem[{Mehrabian et~al.(2020)Mehrabian, Boursier, Kaufmann, and Perchet}]{mehrabian2020practical}
Mehrabian, A.; Boursier, E.; Kaufmann, E.; and Perchet, V. 2020.
\newblock A Practical Algorithm for Multiplayer Bandits When Arm Means Vary among Players.
\newblock In \emph{International Conference on Artificial Intelligence and Statistics}, 1211--1221. {PMLR}.

\bibitem[{Nayyar, Kalathil, and Jain(2016)}]{nayyar2016regretoptimal}
Nayyar, N.; Kalathil, D.; and Jain, R. 2016.
\newblock On Regret-Optimal Learning in Decentralized Multiplayer Multiarmed Bandits.
\newblock \emph{IEEE Transactions on Control of Network Systems}, 5(1): 597--606.

\bibitem[{Pal et~al.(2023)Pal, Suggala, Shanmugam, and Jain}]{pal2023optimal}
Pal, S.; Suggala, A.~S.; Shanmugam, K.; and Jain, P. 2023.
\newblock Optimal Algorithms for Latent Bandits with Cluster Structure.
\newblock \emph{arXiv preprint arXiv:2301.07040}.

\bibitem[{Roijers, Whiteson, and Oliehoek(2015)}]{roijers2015computing}
Roijers, D.~M.; Whiteson, S.; and Oliehoek, F.~A. 2015.
\newblock Computing Convex Coverage Sets for Faster Multi-Objective Coordination.
\newblock \emph{Journal of Artificial Intelligence Research}, 52: 399--443.

\bibitem[{Sankararaman, Ganesh, and Shakkottai(2019)}]{sankararaman2019social}
Sankararaman, A.; Ganesh, A.; and Shakkottai, S. 2019.
\newblock Social Learning in Multi Agent Multi Armed Bandits.
\newblock \emph{Proceedings of the ACM on Measurement and Analysis of Computing Systems}, 3(3): 1--35.

\bibitem[{Scharpff et~al.(2016)Scharpff, Roijers, Oliehoek, Spaan, and {de Weerdt}}]{scharpff2016solving}
Scharpff, J.; Roijers, D.; Oliehoek, F.; Spaan, M.; and {de Weerdt}, M. 2016.
\newblock Solving Transition-Independent Multi-Agent {{MDPs}} with Sparse Interactions.
\newblock In \emph{Proceedings of the {{AAAI}} Conference on Artificial Intelligence}, volume~30.

\bibitem[{Shi et~al.(2020)Shi, Xiong, Shen, and Yang}]{shi2020decentralized}
Shi, C.; Xiong, W.; Shen, C.; and Yang, J. 2020.
\newblock Decentralized Multi-Player Multi-Armed Bandits with No Collision Information.
\newblock In \emph{International Conference on Artificial Intelligence and Statistics}, 1519--1528. {PMLR}.

\bibitem[{Shi et~al.(2021)Shi, Xiong, Shen, and Yang}]{shi2021heterogeneous}
Shi, C.; Xiong, W.; Shen, C.; and Yang, J. 2021.
\newblock Heterogeneous Multi-Player Multi-Armed Bandits: {{Closing}} the Gap and Generalization.
\newblock \emph{Advances in Neural Information Processing Systems}, 34: 22392--22404.

\bibitem[{Stranders et~al.(2012)Stranders, {Tran-Thanh}, Fave, Rogers, and Jennings}]{stranders2012dcops}
Stranders, R.; {Tran-Thanh}, L.; Fave, F. M.~D.; Rogers, A.; and Jennings, N.~R. 2012.
\newblock {{DCOPs}} and Bandits: Exploration and Exploitation in Decentralised Coordination.
\newblock In \emph{Proceedings of the 11th International Conference on Autonomous Agents and Multiagent Systems-Volume 1}, 289--296.

\bibitem[{Szorenyi et~al.(2013)Szorenyi, {Busa-Fekete}, Hegedus, Orm{\'a}ndi, Jelasity, and K{\'e}gl}]{szorenyi2013gossipbased}
Szorenyi, B.; {Busa-Fekete}, R.; Hegedus, I.; Orm{\'a}ndi, R.; Jelasity, M.; and K{\'e}gl, B. 2013.
\newblock Gossip-Based Distributed Stochastic Bandit Algorithms.
\newblock In \emph{International Conference on Machine Learning}, 19--27. {PMLR}.

\bibitem[{Thaker et~al.(2022)Thaker, Malu, Rao, and Dasarathy}]{thaker2022maximizing}
Thaker, P.; Malu, M.; Rao, N.; and Dasarathy, G. 2022.
\newblock Maximizing and Satisficing in Multi-Armed Bandits with Graph Information.
\newblock \emph{Advances in Neural Information Processing Systems}, 35: 2019--2032.

\bibitem[{Thompson(1933)}]{thompson1933likelihood}
Thompson, W.~R. 1933.
\newblock On the Likelihood That One Unknown Probability Exceeds Another in View of the Evidence of Two Samples.
\newblock \emph{Biometrika}, 25(3-4): 285--294.

\bibitem[{Valko et~al.(2014)Valko, Munos, Kveton, and Koc{\'a}k}]{valko2014spectral}
Valko, M.; Munos, R.; Kveton, B.; and Koc{\'a}k, T. 2014.
\newblock Spectral Bandits for Smooth Graph Functions.
\newblock In \emph{International Conference on Machine Learning}, 46--54. {PMLR}.

\bibitem[{Verstraeten et~al.(2020)Verstraeten, Bargiacchi, Libin, Helsen, Roijers, and Now{\'e}}]{verstraeten2020multiagent}
Verstraeten, T.; Bargiacchi, E.; Libin, P.~J.; Helsen, J.; Roijers, D.~M.; and Now{\'e}, A. 2020.
\newblock Multi-Agent Thompson Sampling for Bandit Applications with Sparse Neighbourhood Structures.
\newblock \emph{Scientific reports}, 10(1): 1--13.

\bibitem[{Verstraeten et~al.(2021)Verstraeten, Daems, Bargiacchi, Roijers, Libin, and Helsen}]{verstraeten2021scalable}
Verstraeten, T.; Daems, P.-J.; Bargiacchi, E.; Roijers, D.~M.; Libin, P.~J.; and Helsen, J. 2021.
\newblock Scalable Optimization for Wind Farm Control Using Coordination Graphs.
\newblock In \emph{International Conference on Autonomous Agents and {{MultiAgent}} Systems}, 1362--1370.

\bibitem[{Verstraeten et~al.(2019)Verstraeten, Now{\'e}, Keller, Guo, Sheng, and Helsen}]{verstraeten2019fleetwide}
Verstraeten, T.; Now{\'e}, A.; Keller, J.; Guo, Y.; Sheng, S.; and Helsen, J. 2019.
\newblock Fleetwide Data-Enabled Reliability Improvement of Wind Turbines.
\newblock \emph{Renewable and Sustainable Energy Reviews}, 109: 428--437.

\bibitem[{Wang et~al.(2020{\natexlab{a}})Wang, Proutiere, Ariu, Jedra, and Russo}]{wang2020optimal}
Wang, P.-A.; Proutiere, A.; Ariu, K.; Jedra, Y.; and Russo, A. 2020{\natexlab{a}}.
\newblock Optimal Algorithms for Multiplayer Multi-Armed Bandits.
\newblock In \emph{International Conference on Artificial Intelligence and Statistics}, 4120--4129. {PMLR}.

\bibitem[{Wang et~al.(2020{\natexlab{b}})Wang, Hu, Chen, and Wang}]{wang2020distributed}
Wang, Y.; Hu, J.; Chen, X.; and Wang, L. 2020{\natexlab{b}}.
\newblock Distributed Bandit Learning: {{Near-optimal}} Regret with Efficient Communication.
\newblock In \emph{International Conference on Learning Representations}.

\bibitem[{Wei and Srivastava(2018)}]{wei2018distributed}
Wei, L.; and Srivastava, V. 2018.
\newblock On Distributed Multi-Player Multiarmed Bandit Problems in Abruptly Changing Environment.
\newblock In \emph{{{IEEE}} Conference on Decision and Control ({{CDC}})}, 5783--5788. {IEEE}.

\bibitem[{Wiering et~al.(2000)}]{wiering2000multiagent}
Wiering, M.~A.; et~al. 2000.
\newblock Multi-Agent Reinforcement Learning for Traffic Light Control.
\newblock In \emph{International Conference on Machine Learning}, 1151--1158.

\bibitem[{Xu et~al.(2022{\natexlab{a}})Xu, Wen, Zhao, and Gu}]{xu2022neural}
Xu, P.; Wen, Z.; Zhao, H.; and Gu, Q. 2022{\natexlab{a}}.
\newblock Neural Contextual Bandits with Deep Representation and Shallow Exploration.
\newblock In \emph{International Conference on Learning Representations}.

\bibitem[{Xu et~al.(2022{\natexlab{b}})Xu, Zheng, Mazumdar, Azizzadenesheli, and Anandkumar}]{xu2022langevin}
Xu, P.; Zheng, H.; Mazumdar, E.~V.; Azizzadenesheli, K.; and Anandkumar, A. 2022{\natexlab{b}}.
\newblock Langevin monte carlo for contextual bandits.
\newblock In \emph{International Conference on Machine Learning}, 24830--24850. PMLR.

\bibitem[{Yang, Toni, and Dong(2020)}]{yang2020laplacianregularized}
Yang, K.; Toni, L.; and Dong, X. 2020.
\newblock Laplacian-Regularized Graph Bandits: {{Algorithms}} and Theoretical Analysis.
\newblock In \emph{International Conference on Artificial Intelligence and Statistics}, 3133--3143. {PMLR}.

\bibitem[{Zhang et~al.(2020)Zhang, Zhou, Li, and Gu}]{zhang2020neural}
Zhang, W.; Zhou, D.; Li, L.; and Gu, Q. 2020.
\newblock Neural Thompson Sampling.
\newblock In \emph{International Conference on Learning Representations}.

\bibitem[{Zhang et~al.(2023)Zhang, Qu, Xu, Lin, Chen, and Wierman}]{zhang2023global}
Zhang, Y.; Qu, G.; Xu, P.; Lin, Y.; Chen, Z.; and Wierman, A. 2023.
\newblock Global convergence of localized policy iteration in networked multi-agent reinforcement learning.
\newblock \emph{Proceedings of the ACM on Measurement and Analysis of Computing Systems}, 7(1): 1--51.

\bibitem[{Zhu et~al.(2021)Zhu, Zhu, Liu, and Liu}]{zhu2021federated}
Zhu, Z.; Zhu, J.; Liu, J.; and Liu, Y. 2021.
\newblock Federated Bandit: {{A}} Gossiping Approach.
\newblock In \emph{Abstract Proceedings of the 2021 {{ACM SIGMETRICS}}/{{International}} Conference on Measurement and Modeling of Computer Systems}, 3--4.

\end{thebibliography}

\newpage
\appendix
\onecolumn

\section{Related Works}

There is a rich literature on multi-agent bandits, particularly focusing on scenarios governed by dependency graphs, similar to the one studied in this paper \citep{bargiacchi2018learning, wang2020distributed, landgren2016distributed, sankararaman2019social, agarwal2022multiagent, zhu2021federated, szorenyi2013gossipbased}. These works utilize coordination graphs, where agents can communicate and coordinate their arms. Each agent is represented as a vertex, and edges between agents indicate coordination possibilities. These coordination graph schemes have also been applied to multi-agent Markov decision processes \citep{guestrin2001multiagent, kok2006collaborative, stranders2012dcops, dehauwere2010learning, scharpff2016solving,zhang2023global}. The applications of such frameworks include wind turbine optimization \citep{verstraeten2021scalable}, generative flow models \citep{hayes2022multiobjective}, contextual bandits \citep{han2021budget, elahi2021contextual}, and latent bandits \citep{pal2023optimal}.

Another area of research focuses on collision sensing in bandits, where agents receive no rewards if they choose the same arm. Numerous algorithms have been proposed to find optimal matchings given collision sensing information \citep{magesh2019multiuser, nayyar2016regretoptimal, mehrabian2020practical, shi2021heterogeneous, wang2020optimal}. There are also works that consider scenarios where collision sensing information is not available \citep{shi2020decentralized, lugosi2022multiplayer, huang2022optimal}, as well as cases involving statistical sensing \citep{boursier2020selfish, wei2018distributed, besson2018multiplayer}. In statistical sensing, agents cannot detect collisions directly but can sense the quality of an arm before taking it. Additionally, there are multi-agent studies on information asymmetry \citep{chang2022online}, where agents independently select arms to form a joint arm without communication, aiming to identify the optimal joint arm despite differing feedback for each agent.

In our study, the graph structure is a hypergraph, where vertices represent agents and hyperedges correspond to groups. There are other types of graph structures considered in bandit problems as well. For example, in works such as \cite{gentile2014online, ma2015active, li2016collaborative, yang2020laplacianregularized, gupta2021multiarmed, thaker2022maximizing, deng2023phygcn}, the graphs capture the similarity between arm rewards. By utilizing the graph, the agent can infer arm rewards without extensive sampling. Specifically, they often employ the graph Laplacian, where vertices represent arms and weighted edges indicate similarity. This type of graph structure is also seen in spectral bandits \cite{kocak2021best, yang2020laplacianregularized, valko2014spectral}, where vertices represent signals and edges represent the proximity of means.

Frequentist regret bounds for Thompson sampling have been extensively studied in the single-agent setting. One line of work focuses on the minimax optimality of Thompson sampling \citep{agrawal2017nearoptimal, jin2021mots, jin2022finitetime,jin2023thompson}, analyzing the optimality of gap-independent frequentist regret bounds. Another line of research examines the asymptotic optimality of Thompson sampling \citep{kaufmann2012thompson, korda2013thompson, agrawal2017nearoptimal}, investigating the optimality of gap-dependent frequentist regret bounds.

\section{Table of notations}
In Table \ref{table:notation} we listed the most commonly used notations and their definitions.
\begin{table}
\begin{center}
\caption{Table of notations.}
\label{table:notation}
\begin{tabular}{cp{12cm}}
\toprule
Symbol & Description \\[0.5ex] 
\midrule
$m$ & total number of agents.\\
$K$ & number of arms per agent\\
$A$ & number of joint (global) arms (i.e., $K^m$)\\
$\rho$ & number of (possibly overlapping) groups.\\
$d$   & number of agents in each local group\\
$\localArmSize $ & number of local arms, i.e., $\sum_{e=1}^\rho \sum_{\bm{a}^e:\text{ for some joint arm } \ba} 1 $ \\
$\bm{a}$ & joint arm across $m$ agents. \\
$\bm{1}$ & optimal joint arm.\\
$\bm{A}_t$ & joint arm taken at time $t$.\\
$\mathcal{A}_i$ & set of individual arms for agent $i \in [m]$.\\
$\mathcal{A}^e$ & set of local arms for group $e \in [\rho]$. \\
$\bm{a}^e$ & local arm of $\bm{a}$ corresponding to group $e$.\\
$f^e(\bm{a}^e)$ & stochastic reward for local arm $\bm{a}^e$.\\
$f(\bm{a})$ & sum of local stochastic rewards for arm $\bm{a}$, ie. $\sum_{e=1}^\rho f^e(\bm{a}^e)$.\\
$\mu_{\bm{a}^e}$ & true mean of local reward for local arm $\bm{a}$ corresponding to group $e$.\\
$\mu_{\bm{a}}$ & sum of local true mean rewards for joint arm $\bm{a}$, ie. $\sum_e \mu_{\bm{a}^e}$. \\
$\hat{\mu}_{\bm{a}^e,s}$ & the empirical mean of local arm $\bm{a}^e$ after its $s$-th pull. \\
$\Delta_{\bm{a}^e}$ & the minimum reward gap between joint arm $\bm{a}$ which contains $\bm{a}^e$ and $\bm{1}$, i.e., $\Delta_{\bm{a}^e}=\min\{\Delta_{\bm{a}}\mid \bm{a}^e\in \bm{a}\}$. \\
$\Delta_{\bm{a}}$ & gap between global reward of joint arm $\bm{a}$ and optimal arm $\bm{1}$.\\
$\Delta_{\min}$ & minimum reward gap across all joint arms, i.e. $\min_{\bm{a} \in \times_{i=1}^m\mathcal{A}_i\setminus\{\bm{1}\}} \Delta_{\bm{a}}$.\\
$\Delta_{\max}$ & maximum reward gap across all joint arms, i.e. $\max_{\bm{a} \in \times_{i=1}^m\mathcal{A}_i} \Delta_{\bm{a}}$.\\
$\theta_{\bm{a}^e}(t)$ & sampling posterior corresponding to local arm $\bm{a}^e$ at time $t$.\\
$\theta_{\bm{a}}(t)$ & sum of posteriors for joint arm $\bm{a}$, i.e. $\sum_e \theta_{\bm{a}^e}$.\\
$n_{\bm{a}^e}(t)$ & number of times local arm $\bm{a}^e$ has been pulled up to time $t$.\\
$n_{\bm{a}}(t)$ & number of times joint arm $\bm{a}$ has been pulled up to time $t$.\\
$R_T$ & regret up to time $T$.\\
\bottomrule
\end{tabular}%
\end{center}
\end{table}

\section{Frequentist Regret Analysis of MATS}
In this section, we provide the proof of \Cref{thm:upper bound}.
Following the notations in \citet{verstraeten2020multiagent}, we let $\localArmSize $ be the number of local arms, that is $\localArmSize = \sum_{e \in [\rho]}\sum_{\bm{a}^e \in \cA^e}1$.  For any $r=0,1,\ldots$, let $\delta_r=2^{-(r+2)}$ and 
\begin{align*}
    S_r=\{\bm{a} \mid \Delta_{\bm{a}}\in (2^{-r},2^{-r
    +1}]\}.
\end{align*}
We then decompose the regret as follows, 
\begin{align*}
    R_{T}&=\sum_{r=0}^{\infty}\sum_{t=1}^{T}\sum_{\bm{A}_t\in S_r} \Delta_{\bm{A}_t}. 
\end{align*}
For simplicity, we let $R(S_r)=\sum_{t=1}^{T}\sum_{\bm{A}_t\in S_r} \Delta_{\bm{A_t}}$. In what follows, we focus on bounding the term $R(S_r)$. 
Let $\mathcal{F}_t$ be the history up to time $t$. 
We define $S_r(t)$ to be the set of arms that are not overestimated at time $t$  by the local prior $\theta_{\bm{a}^e}(t)$.
\begin{align*}
    S_{r}(t)=\bigg\{\bm{a} \ \bigg | \ \bm{a}\in S_{r}, \forall e\in [\rho] \ \text{and} \ \bm{a}^e\neq \bm{1}^e, \theta_{\bm{a}^e}(t)\leq  {\mu}_{\bm{a}^e}+\frac{\delta_r}{\rho}\bigg\}.
\end{align*}
Start with regret decomposition 
\begin{align}
    R(S_r)\leq &8\delta_r\sum_{t=1}^{T}\EE[\ind\{\bm{A}_{t}\in \lev\}] \notag \\
    \leq & 8\delta_r\underbrace{\sum_{t=1}^{T}\EE[\ind\{\bm{A}_{t}\in \lev(t)\}]}_{I_1}+ 8\delta_r\underbrace{\sum_{t=1}^{T}\EE[\ind\left\{\bm{A}_{t}\in S_r, \bm{A}_{t}\notin \lev(t)\right\}]}_{I_2},\label{eq:maindecom} 
\end{align}
where the first inequality is because for any $\bm{a}\in S_r$, it holds that $\Delta_{\bm{a}}\leq 8\delta_r$.
Now, we bound these terms separately.

\subsection{Bounding term $I_{1}$} 
This part requires the following supporting lemma.
\begin{lemma}
\label{lem:under-arm1}
Let $\hat{\mu}_{s}$ be the empirical mean of $s$ random variable i.i.d from $\cN(\mu,1)$. Then 
\begin{align*}
    \PP\left(\exists s\in [T], \hat{\mu}_{s}+\sqrt{\frac{2\log (\rho T)}{s}}\leq \mu_1\right)\leq \frac{1}{\rho T}.
\end{align*}
\end{lemma}
Let $\cE_{3}$ be the event that for any $s\leq T$ and $e\in [\rho]$, $\hat{\mu}_{\bm{1}^e,s}\geq \mu_{\bm{1}^e}-\sqrt{\frac{2\log (\rho T)}{s}}$, i.e.,
\begin{align*}
    \cE_{3}=\bigcap_{s\in[T]} \bigcap_{e\in [\rho]}\left\{\hat{\mu}_{\bm{1}^e,s}\geq \mu_{\bm{1}^e}-\sqrt{\frac{2\log (\rho T)}{s}}\right\}.
\end{align*}
From \Cref{lem:under-arm1}, we have
\begin{align*}
    \PP(\cE_{3})\geq 1-\frac{1}{T}. 
\end{align*}
Now, we bound term $I_1$. We have 
\begin{align}
\label{eq:decom-I1}
    I_1&=\sum_{t=1}^{T}\EE\bigg[\ind\{\bm{A}_{t}\in \lev(t)\}\bigg] \notag \\
    & \leq T\cdot \PP((\cE_{3})^c)+\sum_{t=1}^{T}\EE\bigg[\ind\{\bm{A}_{t}\in \lev(t)\}\cdot \ind\{\cE_{3}\}\bigg] \notag \\
    & \leq 1+\sum_{t=1}^{T}\EE\bigg[\ind\{\bm{A}_{t}\in \lev(t)\}\cdot \ind\{\cE_{3}\}\bigg].
\end{align}
We first define some useful notations. Suppose $J$ is a subset of $2^{[\rho]}$. Subsequently, we define $S_{rJ}$ as a specific subset of $S_r$, characterized by the condition that every arm in $S_{r}$ exclusively shares local arms of $\bm{1}$ for groups in $J$, i.e.,
\begin{align*}
   S_{rJ}=\big\{\bm{a} \ \big | \ \bm{a}\in S_{r}, \forall e\in J,  \bm{a}^e=\bm{1}^e \ \text{and} \ \forall e\in [\rho]\setminus J, \bm{a}^e\neq \bm{1}^e\big\}.
\end{align*}
We also define 
\begin{align*}
    S_{rJ}(t)=\{\bm{a} \mid \bm{a}\in S_{rJ}, \bm{a}\in S_{r}(t)\}.
\end{align*}
Let $\cE_{1}(t)$ be the event that the local arms of the optimal joint arm are not underestimated. Namely,
\begin{align*}
    \mathcal{E}_1(t) = \bigcap_{e=1}^{\rho} \left\{\theta_{\bm{1}^e}(t)\geq \mu_{\bm{1}^e}-\frac{\delta_r}{\rho}\right\}.
\end{align*}
Similar to $S_{rJ}(t)$, we define
\begin{align*}
    \mathcal{E}_{1J}(t)=\bigcap_{e:e\in [\rho]\setminus J} \left\{\theta_{\bm{1}^e}(t)\geq \mu_{\bm{1}^e}-\frac{\delta_r}{\rho}\right\}.
\end{align*}
Additionally, we introduce $\cE_{2J}(t)$, which represents the event where arms that are associated with local arms labeled by $\bm{1}$ in $[\rho]\setminus J$ are not overestimated.
\begin{align*}
    \cE_{2J}(t)=\bigcap_{\bm{a}:\bm{a}\in S_{rJ}}\left\{ \forall e\in [\rho]\setminus J, \theta_{\bm{a}^e}(t)\leq \mu_{\bm{a}^e}+\frac{\delta_r}{\rho}\right\}.
\end{align*}
Let $\bm{A}_{t}(J)=\max_{\bm{a}\in S_{rJ}} \theta_{\bm{a}}(t)$. For term $\ind\{\bm{A}_{t}\in \lev(t)\}$, we bound it as follows. 
\begin{align}
    2^{\rho}\cdot \PP(\bm{A}_t =\bm{1} \mid\mathcal{F}_{t-1}) &\geq \sum_{J\subseteq 2^{[\rho]}} \PP\bigg(\forall \bm{a} \in S_{rJ}, \sum_{e\in [\rho]\setminus J} \theta_{\bm{1}^e}(t)\geq \sum_{e\in [\rho]
    \setminus J} \theta_{\bm{a}^e}(t) \ \bigg | \ \mathcal{F}_{t-1} \bigg) \notag \\
    &\geq \sum_{J\subseteq 2^{[\rho]}}  \PP(\bm{A}_{t}(J)\in S_{rJ}(t), \cE_{1J}(t) \mid\mathcal{F}_{t-1}) \notag \\
    &\geq \sum_{J\subseteq 2^{[\rho]}} \bigg(\PP(\bm{A}_{t}\in S_{rJ}(t) \mid  \cF_{t-1})\cdot \PP(\cE_{1}(t)\mid \cF_{t-1})\bigg) \notag \\
     &\geq \PP(\cE_1(t)\mid \cF_{t-1}) \cdot \sum_{J\subseteq 2^{[\rho]}}\PP(\bm{A}_{t}\in S_{rJ}(t) \mid \cF_{t-1}) \notag \\
    &=\PP(\cE_1(t)\mid \cF_{t-1}) \cdot  \PP(\bm{A}_{t}\in S_{r}(t)\mid \cF_{t-1}).\label{eq:1-n}
\end{align}
The first inequality is due to the fact that if there exists $\bm{a} \in S_{rJ}, \sum_{e\in [\rho]\setminus J} \theta_{\bm{1}^e}(t)< \sum_{e\in [\rho]\setminus J} \theta_{\bm{a}^e}(t)$, then $\theta_{\bm{1}}(t)<\theta_{\bm{a}}(t)$ \footnote{This follows from the fact that $\theta_{\bm{a}^e} = \theta_{\bm{1}^e}$ for $e \in J$ when $\bm{a} \in S_{rJ}$.}, which means $\bm{A}_{t}\neq \bm{1}$. The second inequality is due to the fact that if the events $\bm{A}_{t}(J)\in S_{rJ}(t)$ and 
$\cE_{1J}(t)$ occur, then for any $\bm{a}\in S_{rJ}$,
\begin{align*}
    \sum_{e\in [\rho]\setminus J}\theta_{\bm{1}^e}(t)& \geq \sum_{e\in [\rho]\setminus J} \bigg(\mu_{\bm{1}^e}-\frac{\delta_r}{\rho}\bigg) \notag \\
    & \geq \sum_{e\in [\rho]\setminus J} \bigg(\mu_{(\bm{A}_t(J))^e}+\frac{\delta_r}{\rho}\bigg) \notag \\
    & \geq  \sum_{e\in [\rho]\setminus J}\theta_{(\bm{A}_t(J))^e}(t) \notag \\
    & \geq \sum_{e\in [\rho]\setminus J} \theta_{\bm{a}^e}(t),  \qquad \qquad \text{which means $\bm{A}_{t}\in S_{rJ}(t)$ occur}
\end{align*} 
where the first inequality is from the definition of the event $\cE_{1J}(t)$ and the second inequality is due to $\bm{A}_t(J)\in S_{rJ}$ and $\Delta_{\bm{A}_t(J)}\geq 4\delta_r$. The third inequality of \eqref{eq:1-n} is due to $\bm{A}_{t}(J)\in S_{rJ}(t)$, as well as $\cE_{1J}(t)$ and $\bm{A}_{t}(J)\in S_{rJ}(t)$ being conditionally independent given $\cF_{t-1}$. Rearranging \eqref{eq:1-n}, we obtain
\begin{align}
    \PP(\bm{A}_{t}\in S_r(t)\mid \cF_{t-1})\leq  \frac{2^{\rho}\cdot \PP(\bm{A}_{t}=1\mid \cF_{t-1})}{\PP(\cE_{1}(t)\mid \cF_{t-1})}.\label{eq:tor1}
\end{align}
 Let $\tau=\frac{8\rho^2 \log^2 (T\localArmSize)}{\delta_r^2}$. Continuing on inequality \eqref{eq:decom-I1}, we further  bound $I_1$ as follows.
\begin{align}
    I_1&\leq 1+\sum_{t:t\in [T]}\EE\bigg[\ind\{\bm{A}_{t}\in \lev(t)\}\ind\{\cE_{3}\}\bigg] \notag \\
     & \leq 1+\sum_{t:n_{\bm{1}}(t)\leq \tau}\EE\bigg[\ind\{\bm{A}_{t}\in \lev(t)\}\ind\{\cE_{3}\}\bigg]+ \sum_{t:n_{\bm{1}}(t)>\tau}\EE\bigg[\ind\{\bm{A}_{t}\in \lev(t)\}\bigg] \notag \\
       &\leq 1+\underbrace{\sum_{t: n_{\bm{1}}(t)\leq \tau}\EE\bigg( \frac{2^{\rho}}{\PP(\cE_{1}(t)\mid \cF_{t-1})} \ind\{\cE_{3}\}\cdot \ind[\bm{A}_{t}=1] \bigg)}_{I_{11}} \notag \\
       &\qquad +\underbrace{\sum_{t:n_{\bm{1}}(t)>\tau}\EE\bigg[\ind\{\bm{A}_{t}\in \lev(t)\}\bigg]}_{I_{12}}\label{eq:boundI1},
\end{align}
where the last inequality is due to \eqref{eq:tor1}.

We bound $I_{11}$ and $I_{12}$ separately. \\

\textit{Bounding term $I_{11}$.} We have
\begin{align}
      \frac{\ind\{\cE_{3}\} }{\PP(\cE_{1}(t)\mid \cF_{t-1})}  
    = & \prod_{e=1}^{\rho}  \frac{1}{\PP\left(\theta_{\bm{1}^{e}}(t)\geq \mu_{\bm{1}^e}-\frac{\delta_r}{\rho} \ \bigg | \ \hat{\mu}_{\bm{1}^e}(t)\geq \mu_{\bm{1}^e}-\sqrt{\frac{2\log T}{n_{\bm{1}^e}(t)}}\right)} \notag \\
    \leq & \prod_{e=1}^{\rho}  \frac{1}{\PP\left(\theta_{\bm{1}^{e}}(t)\geq \mu_{\bm{1}^e} \ \bigg | \ \hat{\mu}_{\bm{1}^e}(t)\geq \mu_{\bm{1}^e}-\sqrt{\frac{2\log T}{n_{\bm{1}^e}(t)}}\right)}. \label{eq:bound-I11-first}
\end{align}
Note that with probability $\epsilon$, $\theta_{\bm{1}^e}(t) \sim \cN(\hat\mu_{\bm{1}^e}(t), \log T/(n_{\bm{1}^e}(t)+1))$.
By letting $z=\sqrt{2(n_{\bm{1}^e}(t)+1)/n_{\bm{1}^e}(t)}\leq 4$ and $\sigma=\sqrt{{\log T}/({n_{\bm{1}^e}(t)}+1)}$ in Lemma \ref{lem:antiGAU}, we have
\begin{align}
 & \PP\left(\theta_{\bm{1}^{e}}(t)\geq \mu_{\bm{1}^e} \ \bigg | \ \hat{\mu}_{\bm{1}^e}(t)\geq \mu_{\bm{1}^e}-\sqrt{\frac{2\log T}{n_{\bm{1}^e}(t)}}\right) \notag \\
  &\geq    \PP\left(\theta_{\bm{1}^{e}}(t)\geq \hat{\mu}_{\bm{1}^e}(t)+\sqrt{\frac{2\log T}{n_{\bm{1}^e}(t)}}\right) \notag \\
  &\geq  \frac{\epsilon}{3\times 10^4}.\label{eq:bound-I11-fin}
\end{align}
By substituting \eqref{eq:bound-I11-fin} to  \eqref{eq:bound-I11-first}, we obtain 
\begin{align*}
    \frac{1}{\PP(\cE_{1}(t)\mid \cF_{t-1})} \ind\{\cE_{3}\} \leq (3\times 10^4/\epsilon)^{\rho}.
\end{align*}
Finally, we can bound $I_{11}$ as
\begin{align*}
    I_{11}&=\sum_{t: n_{\bm{1}}(t)\leq \tau}\EE\bigg( \frac{2^{\rho}}{\PP(\cE_{1}(t)\mid \cF_{t-1})} \ind\{\cE_{3}\}\cdot \ind[\bm{A}_{t}=1] \bigg) \notag \\
    & \leq \sum_{t: n_{\bm{1}}(t)\leq \tau}\EE\bigg((6\times 10^4/\epsilon)^{\rho} \cdot \ind[\bm{A}_{t}=1] \bigg) \notag \\
    &\leq \tau \cdot (6\times 10^4/\epsilon)^{\rho}.
\end{align*}
\textit{Bounding term $I_{12}$.}
If $n_{\bm{1}}(t)$ is pulled more than $\tau$ times, the posterior sample of each local arm $\bm{1}^e$ is concentrated on $\mu_{\bm{1}^e}$ and thus we can prove the following lemma.  
\begin{lemma}
\label{lemma:B.2}
We have
\begin{align*}
    \sum_{t:n_{\bm{1}}(t)>\tau} \EE \bigg[\ind\{\bm{A}_{t}\in S_r(t)\} \bigg]\leq 2.
\end{align*}
\end{lemma}
From the above lemma, $I_{12}\leq 2$.
By combining the results  of term $I_{11}$ and $I_{12}$ together and then substituting them into \eqref{eq:boundI1}, we obtain
\begin{align}
\label{eq:res-I1}
    I_1\leq 3+\tau\cdot (6\times 10^4/\epsilon)^{\rho}.
\end{align}

\subsection{Bounding term $I_2$}
Let $\cE$ be the event that if any local arm $\bm{a}^{e}$ has been pulled more than $\tau$ times at $t$, then for this arm (which belongs in $S_r$) and any $t$, we have
$|\theta_{\bm{a}^e}(t)-{\mu}_{\bm{a}^e}|\leq \frac{\delta_r}{\rho}$, i.e., 
\begin{align*}
\cE(t)=\bigcap_{e=1}^{\rho}\bigcap_{\bm{a}^{e}: n_{\bm{a}^e}(t)\geq \tau}\left\{|\theta_{\bm{a}^e}(t)-\mu_{\bm{a}^{e}}|<\frac{\delta_r}{\rho}\right\}.
 \end{align*}
The following lemma shows that $\cE$ happens with high probability. 
\begin{lemma}
\label{eq:lem-b3}
We have
\begin{align*}
   \Pr(\cE(t))\geq 1-\frac{4}{T}.
\end{align*}
Let $Y_r(t)$ be the indicator on the event such that the joint arm taken at time $t$ belongs in $S_r$ and there exists a group such that the local arm of this group pulled less than $\tau$ times, i.e., 
\begin{align*}
    Y_{r}(t)=\ind\{\bm{A}_{t}=\bm{a}, \bm{a}\in S_r,  \exists e\in [\rho] \ \text{s.t.} \ n_{\bm{a}^e}(t)<\tau \}.
\end{align*}
Recall that $\localArmSize $ is the total number of local arms. Intuitively, if $Y_{r}(t)=1$, then 
there is at least one local arm with $n_{\bm{a}^e}(t)<\tau$ that has been pulled one more time at round $t$. Since there are a total $\localArmSize $ local arms, after at most $\tau\cdot \localArmSize $ number of rounds $t$ such that $Y_r(t) = 1$, the number of pulls of every local arm exceeds $\tau$. Therefore, $Y_r(
t)$ can only be $1$ for at most $\tau \cdot \localArmSize$ time steps and is $0$ for the rest of them. 
Therefore,
\begin{align*}
    \sum_{t=1}^{T} Y_{r}(t)\leq \tau\cdot A_r,
\end{align*}
where $A_r$ is the  number of local arms $\bm{a}^e$ which satisfies $\exists \bm{a}\in S_r$, $\bm{a}^e\in \bm{a}$.
Finally, we can bound $I_2$ as follows, 
\begin{align}
   I_2=&\sum_{t=1}^{T}\EE\bigg[\ind\{\bm{A}_{t}\in S_r, \bm{A}_{t}\notin S_{r}(t) \}\bigg] \notag \\
    = &\sum_{t=1}^{T}\EE\bigg[\ind\{\bm{A}_{t}\in S_r, \bm{A}_{t}\notin S_{r}(t), \cE(t)\}\bigg]+\sum_{t=1}^{T}\EE\bigg[\ind\{\bm{A}_{t}\in S_r, \bm{A}_{t}\notin S_{r}(t), (\cE(t))^c\}\bigg] \notag \\
    \leq & \sum_{t=1}^{T} Y_{r}(t)+ \sum_{t=1}^{T} \PP((\cE(t))^c) \notag \\
    \leq & \tau \cdot A_r+4\label{eq:bound-I2}.
\end{align}
\end{lemma}
We use the following facts in the first inequality. From the definition of $S_{r}(t)$, we have for any $\bm{A}_{t}\in S_r\backslash S_r(t)$, %
\begin{align}
\label{eq:last-1}
    \exists e\in [\rho], \bm{A}_t^e\neq \bm{1}^e, \theta_{\bm{A}_t^e}(t)>\mu_{\bm{A}_t^e}+ \frac{\delta_r}{\rho}.
\end{align}
Moreover, if $\cE(t)$ happens, \eqref{eq:last-1} implies that 
\begin{align}
\label{eq:last-2}
    \exists e\in [\rho], n_{\bm{A}_t^e}(t)<\tau,
\end{align}
which means $Y_{r}(t)=1$. Therefore, $\ind\{\bm{A}_{t}\in S_r, \bm{A}_{t}\notin S_{r}(t), \cE\}\leq Y_{r}(t)$.

\subsection{Putting Everything Together} 
By combining \eqref{eq:res-I1} and \eqref{eq:bound-I2} together and then substituting them into \eqref{eq:maindecom}, we obtain that there exists some universal constant $C$ such that 
\begin{align*}
    R(S_i)&\leq 8\delta_r\cdot \tau \Big(A_r+\big(6\cdot 10^4/\epsilon\big)^\rho \Big)+80\delta_r.
\end{align*}
Therefore, by letting $\Delta_{\min}=\min_{\bm{a}: \bm{a}\neq \bm{1}} \Delta_{\bm{a}} $ and $\Delta_{\max}=\max_{\bm{a}} \Delta_{\bm{a}} $, then there exists some universal constant $C$ such that
\begin{align*}
    R_{T}&=\sum_{r} R({S_r}) \notag \\
        & \leq \sum_{r}\bigg(8\delta_r\cdot \tau \bigg(A_r+\big(6\cdot 10^4/\epsilon\big)^\rho \bigg)+80\delta_r\bigg)\notag \\
        & =  \sum_{r}\bigg(\frac{64\rho^2 \log^2 (T\localArmSize)}{\delta_r}\cdot \bigg(A_r+\big(6\cdot 10^4/\epsilon\big)^\rho \bigg)+80\delta_r\bigg)\notag \\
        & \leq \frac{C(C/\epsilon)^{\rho}\rho^2 \log^2 (T\localArmSize)}{\Delta_{\min}}+C\Delta_{\max}+ \sum_{i} \frac{C\rho^2 A_r \log^2 (T\localArmSize)}{\delta_r},\label{eq:mini}
\end{align*}
where in the last inequality $C\Delta_{\max}$ is from $\sum_{r}80\delta_r$.
We let $S(\bm{a^{e}})=\{r \mid \exists \bm{a} \in S_r, \bm{a}^e\in \bm{a}\}$ and $\Delta_{\bm{a}^e}=\min\{\Delta_{\bm{a}}\mid \bm{a}^e\in \bm{a}\}$.
For the last term $\sum_{i}\frac{C\rho^2 A_r \log^2 (T\localArmSize)}{\delta_r}$, we can rewrite it in the following way. 
\begin{align}
    \sum_{r} \frac{C\rho^2 A_r \log^2 (T\localArmSize)}{\delta_r}&=\sum_{e\in [\rho]}\sum_{\bm{a}^e\in \cA
    ^{e}\setminus \{\bm{1}^e\}}\sum_{r\in S(\bm{a}^e)} \frac{C\rho^2 \log^2 (T\localArmSize)}{\delta_r}\notag \\
    &\leq 8\sum_{e\in [\rho]}\sum_{\bm{a}^e\in \cA^{e}\setminus \{\bm{1}^e\}} \frac{C\rho^2 \log^2 (T\localArmSize)}{\Delta_{\bm{a}^e}}.
\end{align}
Therefore, there exists some universal constant $C_1$ such that 
\begin{align*}
    R_{T}\leq \frac{C_1(C_1/\epsilon)^{\rho}\rho^2 \log^2 (T\localArmSize)}{\Delta_{\min}}+C_1\sum_{e\in [\rho]}\sum_{\bm{a}^e\in \cA^{e}\setminus \{\bm{1}^e\}}\frac{\rho^2 \log^2 (T\localArmSize)}{\Delta_{\bm{a}^e}}+C_1\Delta_{\max}.
\end{align*}
\textbf{Worst Case Regret.}
We have that there exists some universal constant $C_2$ such that
\begin{align*}
    R_{T}&\leq \min \left\{T\Delta_{\min}, \frac{C_1(C_1/\epsilon)^{\rho}\rho^2 \log^2 (T\localArmSize)}{\Delta_{\min}}+C_1\Delta_{\max}+ \sum_{i} \frac{C_1\rho^2 A_i \log^2 (T\localArmSize)}{\delta_i}\right\} \notag \\
    &\leq \min \left\{T\Delta_{\min}, \frac{C_1(C_1/\epsilon)^{\rho}\rho^2 \log^2 (T\localArmSize)}{\Delta_{\min}}+C_1\Delta_{\max}+ \frac{C_1\rho^2 \localArmSize  \log^2 (T\localArmSize)}{\Delta_{\min}}\right\}.
\end{align*}
Note that when $\Delta_{\min}< \sqrt{\frac{C_1((C_1/\epsilon)^{\rho}+\localArmSize)\rho^2 \log^2 (T\localArmSize)}{T}}$, we have
\begin{align*}
    T\Delta_{\min}\leq \sqrt{{C_1((C_1/\epsilon)^{\rho}+\localArmSize)\rho^2 T \log^2 (T\localArmSize)}},
\end{align*}
and when  $\Delta_{\min}\geq \sqrt{\frac{C_1((C_1/\epsilon)^{\rho}+\localArmSize)\rho^2 \log^2 (T\localArmSize)}{T}}$,  we have
\begin{align*}
    \frac{C_1(C_1/\epsilon)^{\rho}\rho^2 \log^2 (T\localArmSize)}{\Delta_{\min}}+ \frac{C_1\rho^2 \localArmSize  \log^2 (T\localArmSize)}{\Delta_{\min}}\leq \sqrt{{C_1((C_1/\epsilon)^{\rho}+\localArmSize)\rho^2 T \log^2 (T\localArmSize)}}.
\end{align*}
Therefore, the worst-case regret of \algname\ is bounded as follows.
\begin{align*}
    R_{T}\leq C_2\Delta_{\max}+ C_2\rho \sqrt{\left((C_2/\epsilon)^{\rho}+\localArmSize \right) T \log^2 (T\localArmSize)}.
\end{align*}

\section{Proof of Supporting Lemmas}
In this section, we provide the proofs for the technical lemmas we used in our main paper and the proof of the main theorem.

\subsection{Proof of Lemma \ref{lemma:variable_elimination_complexity}}
Let $G_1,\ldots, G_\rho$ be the sets of agents belonging to group $e\in[\rho]$, respectively. To simplify notation, we denote $\cG$ as the hypergraph where $G_1,\ldots, G_\rho$ represent the hyperedges. Then, $\localArmSize = \sum_{e=1}^\rho \prod_{i \in G_e} |\mathcal{A}_i|$. We can prove the desired results by induction on the number of agents.

For the base case, when there is only one agent, the coordination graph consists of a single vertex with degree 0. Since we have $|\mathcal{A}_1|$ arms to optimize over, the complexity is $|\mathcal{A}_1|$, completing the base case.

For the induction step, let's assume that the result in \Cref{lemma:variable_elimination_complexity} holds for a hypergraph with $n-1$ vertices. Now we consider a hypergraph $\cG$ with $n$ agents. Similar to \eqref{eq:optimize}, we decompose the optimization problem as follows:
\begin{align}
\max_{\bm{a}} f(\bm{a}) &= \max_{\bm{a}}\sum_{e=1}^\rho f^e (\bm{a}^e) \notag\\
&= \max_{\bm{a}} \bigg[ \underbrace{\sum_{\bm{a}^e \in \bm{a}: a_n \notin \bm{a}^e} f^e (\bm{a}^e)}_{I_1} + \underbrace{\max_{\bm{a}^e: a_n \in \bm{a}^e} \sum_{\bm{a}^e \in \bm{a}: a_n \in \bm{a}^e} f^e (\bm{a}^e)}_{I_2}\bigg],\label{eq:optimize_induction}
\end{align}
where $a_n$ represents an individual arm of agent $n$. Without loss of generality, let's assume that this agent belongs to $l$ different groups, namely $G_1, G_2, \ldots, G_l$. Note that Term $I_2$ in \eqref{eq:optimize_induction} is a function that involves agents within this group. All the values for this function can be found in $\sum_{e=1}^l \prod_{i \in G_e} |\mathcal{A}_i|$ searches.

Now let's consider the graph $\cG' = \cG \backslash {n}$, which has the same coordination graph as $\cG$ but with vertex $n$ and all its connections removed. Let $G_1', \ldots, G_\rho'$ be the new groups resulting from this reduction. It is clear that $\forall e \in [\rho], |G_e| \geq |G_e'|$. Since $\cG'$ has only $n-1$ vertices, by the inductive hypothesis, the complexity of finding the maximum of $I_1$ (using all of the values already given from computing $I_2$) is $\sum_{e=1}^\rho \prod_{i \in G_e'} |\mathcal{A}_i| \leq \sum_{e=1}^\rho \prod_{i \in G_e} |\mathcal{A}_i| \leq \localArmSize$. Adding these two upper bounds completes our induction.

\subsection{Proof of Lemma \ref{lem:under-arm1}}
Note that
\begin{align*}
    \PP\left(\exists s\in [T], \hat{\mu}_{s}+\sqrt{\frac{2\log (\rho T)}{s}}\leq \mu_1\right)
    \leq& \sum_{s=1}^{T} \PP\left( \hat{\mu}_{s}+\sqrt{\frac{2\log (\rho T)}{s}}\leq \mu_1\right) \notag \\
    \leq & \sum_{s=1}^{T} e^{-2\log (\rho T)} \notag \\
    = & \frac{1}{\rho^2 \cdot T},
 \end{align*}
 where the second inequality is due to Lemma \ref{lem:pink}.

\subsection{Proof of Lemma \ref{lemma:B.2}}
Each local arm of $\bm{1}$ has been pulled more than $\tau$ times if $n_{\bm{1}}(t)\geq \tau$. Besides,
\begin{align}
         &\PP\left(\exists s\geq \tau, \hat{\mu}_{\bm{1}^e,s}\leq \mu_{\bm{1}^e}-\delta_r/(2\rho) \right) \notag \\
    &\leq  \exp\bigg(- \frac{\tau\delta_r^2}{8\rho^2} \bigg) \notag\\
    &\leq \frac{1}{T\rho}. \label{eq:fin1}
    \end{align}
Applying union bound for all $e\in [\rho]$, we have
\begin{align}
    \PP\bigg(\bigcup_{e\in [\rho]} \left\{\exists s\geq \tau, \hat{\mu}_{\bm{1}^e,s}\leq \mu_{\bm{1}^e}-\frac{\delta_r}{2\rho}\right\}  \bigg)\leq \frac{1}{T}.\label{eq:fin2} 
\end{align}
Let $E_1=\cap_{e\in [\rho]} \{\forall s\geq \tau, \hat{\mu}_{\bm{1}^e,s}> \mu_{\bm{1}^e}-\delta_r/(2\rho)\}$. We have
\begin{align}
     &\sum_{t: n_{\bm{1}^e}(t)\geq \tau}\PP\bigg(\exists e\in [\rho]: \theta_{\bm{1}^e}(t)\leq \mu_{\bm{1}^e}- \frac{\delta_r}{\rho} , E_1  \bigg) \notag\\
    &\leq  T\rho \cdot \PP\bigg( \theta_{\bm{1}^e}(t)\leq \mu_{\bm{1}^e}- \frac{\delta_r}{\rho} \ \bigg | \ n_{\bm{1}^e}(t)\geq \tau, \hat{\mu}_{\bm{1}^e}(t)>\mu_{\bm{1}^e}-\frac{\delta_r}{2\rho}\bigg) \notag \\ 
       &\leq  T\rho \cdot \PP\bigg( \theta_{\bm{1}^e}(t)\leq \hat{\mu}_{\bm{1}^e}(t)- \frac{\delta_r}{2\rho} \ \bigg | \ n_{\bm{1}^e}(t)\geq \tau \bigg) \notag \\
    &\leq  \frac{T\rho}{2} \cdot \exp\bigg( -\frac{\delta_r^2 \cdot \tau}{8\log T}\bigg) \notag \\
    &\leq  \frac{1}{2},\label{eq:last-rel1}
 \end{align}
where the third inequality is due to the fact that with probability $1-\epsilon$, $\theta_{\bm{1}^e}(t)=\hat{\mu}_{\bm{1}^e}(t)\geq \hat{\mu}_{\bm{1}^e}(t)-\delta_r/(2\rho)$ and with probability $\epsilon$, $\theta_{\bm{1}^e}(t)\sim \cN(\hat{\mu}_{\bm{1}^e}(t), \log T/(n_{\bm{1}^e}(t)+1))$ and then we use Lemma \ref{lem:antiGAU} to bound the probability of $\theta_{\bm{1}^e}(t)\leq \hat{\mu}_{\bm{1}^e}(t)- \frac{\delta_r}{2\rho}$. %
Finally,  
\begin{align*}
     \sum_{t:n_{\bm{1}}(t)>\tau} \EE \bigg[\ind\{\bm{A}_{t}\in S_r(t)\} \bigg]     
    \leq   &T\cdot E_{1}^{c} +   \sum_{t:n_{\bm{1}}(t)\geq \tau} \EE \bigg[\ind\{\bm{A}_{t}\in S_r(t),E_1\} \bigg] \notag \\
    \leq   & 1+  \sum_{t:n_{\bm{1}}(t)>\tau} \PP \bigg(\exists e\in [\rho]: \theta_{\bm{1}^e}(t)\leq \mu_{\bm{1}^e}- \frac{\delta_r}{\rho}, E_1 \bigg) \notag \\
    \leq & 2,
\end{align*}\
where the third inequality is due to \eqref{eq:fin2} and the last inequality is due to \eqref{eq:last-rel1}.

\subsection{Proof of Lemma \ref{eq:lem-b3}}
Recall that
\begin{align*}
    \cE(t)=\bigcap_{e=1}^{\rho}\bigcap_{\bm{a}^{e}:n_{\bm{a}^e}(t)\geq \tau}\left\{|\theta_{\bm{a}^e}(t)-\mu_{\bm{a}^{e}}|<\frac{\delta_r}{\rho} \right\}.
\end{align*}
Similar to \eqref{eq:fin2}, we have 
\begin{align}
     \PP\bigg(\bigcup_{e\in [\rho]}\bigcup_{\bm{a}^e: n_{\bm{a}^e}(t)\geq \tau} \left\{ |\hat{\mu}_{\bm{a}^e,s}- \mu_{\bm{a}^e} |> \frac{\delta_r}{2\rho}\right\}  \bigg)\leq \frac{2}{T}.\label{eq:fin-l1}
\end{align}
Let $E_2=\cap_{e\in [\rho]}\cap_{\bm{a}^e: n_{\bm{a}^e}(t)\geq \tau} \{ |\hat{\mu}_{\bm{a}^e}(t)- \mu_{\bm{a}^e}| \leq \delta_r/(2\rho)\}$.
Now, we have
\begin{align*}
    \PP(\cE(t))&\geq \PP(\cE(t),E_2) \notag \\
               &= \PP(\cE(t)\mid E_2) \PP(E_2) \notag \\
               &\geq  \bigg(1-\frac{2}{T} \bigg)\cdot \bigg(1-\PP\bigg(\exists \bm{a}^{e}: n_{\bm{a}^e}(t)\geq \tau, |\theta_{\bm{a}^e}(t)-\mu_{\bm{a}^e}|< \frac{\delta_r}{\rho} \ \bigg | \ E_2  \bigg) \bigg) \notag \\
               & \geq  \bigg(1-\frac{2}{T} \bigg)\cdot \bigg(1-\PP\bigg(\exists \bm{a}^{e}: n_{\bm{a}^e}(t)\geq \tau, |\theta_{\bm{a}^e}(t)-\mu_{\bm{a}^e}|< \frac{\delta_r}{\rho} \ \bigg | \  |\hat{\mu}_{\bm{a}^e}(t)-\mu_{\bm{a}^e}|< \frac{\delta_r}{2\rho}  \bigg) \bigg) \notag \\
               &\geq  \bigg(1-\frac{2}{T} \bigg)\cdot \bigg(1-\localArmSize \PP\bigg( n_{\bm{a}^e}(t)\geq \tau, |\theta_{\bm{a}^e}(t)-\hat{\mu}_{\bm{a}^e}(t)|< \frac{\delta_r}{2\rho}  \bigg) \bigg) \notag \\
               & \geq \bigg(1-\frac{2}{T} \bigg)\cdot \bigg(1-\localArmSize \exp \bigg(-\frac{\tau \delta_r^2}{8\rho^2} \bigg)\bigg) \notag \\
               & \geq 1-\frac{4}{T},
\end{align*}
where the third inequality is due to \eqref{eq:fin-l1} and the fifth inequality is due to Lemma \ref{lem:antiGAU}.

\section{Proof of the Lower Bound of MAMAB with Coordination Hypergraph}
In this section, we derive a worst-case lower bound on the MAMAB considered in this paper. 
 \begin{proof}[Proof of \Cref{thm:lower_bound}]
 For any given group structure on the agents, suppose group $1$ was the largest group. It follows that the number of local actions for this group satisfies $|\mathcal{A}^1| \geq \frac{\localArmSize}{\rho}$. Furthermore, suppose that local arms in the other groups $[\rho]\backslash 1$ possess a mean reward of $0$, allowing only the local arms in group $1$ to have a nonzero mean reward.
 Moreover, let $\bm{1}^{1}$ signify the local arm in group 1 possessing the highest mean. Since only group $1$ has nonzero local mean rewards, the regret associated with pulling a suboptimal arm $\bm{a}$ is calculated as follows:
\begin{align*}
\sum_{e\in [\rho]} (\mu_{\bm{1}^e}-\mu_{\bm{a}^e} )=\mu_{\bm{1}^1}-\mu_{\bm{a}^1}.
\end{align*}

With these considerations, the problem is effectively simplified to a $K$-armed bandit problem with $K = |\mathcal{A}^1|\geq \frac{\localArmSize}{\rho}$. As it is widely established, the worst-case lower bound for a $K$-armed bandit is $\Omega(\sqrt{KT})$ \citep{auer2002nonstochastic}. Hence, a MAMAB problem instance can be found where the regret is equivalent to $\Omega(\sqrt{KT})=\Omega(\sqrt{ T\localArmSize/\rho})$.
\end{proof}

\begin{proof}[Proof of Theorem \ref{thm:lcb}]
Consider the following bandit instance: We have $\rho$ groups. For each group, we have $L+1$ local arms, where $L$ local arms have mean $X$ and 1 local arm has mean $X+\Delta$. Now, we construct the joint arm set, where we only have in total $L^{\rho}+1$ joint arms, with $L^{\rho}$ joint arms having mean $\rho X$ and one joint arm having  mean $\rho(X+\Delta)$. For each $L^{\rho}$ joint arms that have mean $\rho X$, it is composed by $\rho$ local arms with mean $X$. For the one joint arm with mean $\rho (X+\Delta)$, it is composed by $\rho$ local arms with mean $X+\Delta$. It is easy to see the joint arm with mean $\rho (X+\Delta)$ is the optimal one. We denote the optimal joint arm as $\textbf{1}$. All local arms follow the Gaussian reward distribution with unit variance. We use a normal distribution $\mathcal N(0,1)$ as the prior for all local arms.
We aim to compute the number of pulls of suboptimal joint arms before pulling the optimal joint arm for the first time. The posterior distribution of the optimal joint arm is $\mathcal{N}(0,\rho)$. From the concentration bound of Gaussian (Lemma F.1), we obtain that the probability of the sample from the optimal arm $\textbf{1}$ being larger than $\rho$ is bounded by
$$
\mathbb{P}(\theta_{\textbf{1}}\geq \rho)\leq \frac{e^{-{\rho}/{2}}}{2}.
$$
When we pull a local arm $i$ with mean $X$ for $s$ times, then its posterior distribution is $\mathcal{N}(s\hat{\mu}_{i,s}/(s+1),1/(s+1))$, where $\hat{\mu}_{i,s}$ is the empirical mean of local arm $i$ after its $s$-th pull.

Assume $X>3$. Note that $\hat{\mu}_{i,s}\sim \mathcal{N}(X,1/s)$. Applying the maximum inequality (Lemma F.2), we obtain
$$
\mathbb{P}(\exists s>1, \hat{\mu}_{i,s}<2)\leq 1/\sqrt{e}.
$$
We let $y_e$ be the number of local arms with mean $X$ in group $e$ such that its empirical mean is always larger than 2. By Hoeffding's inequality,
$$
\mathbb{P}(y_e\leq L/(2\sqrt{e})) \leq \exp \bigg(-\frac{L}{2e} \bigg).
$$
We let $E_e$ be the event $y_e> L/(2\sqrt{e})$ and $E:=\cap_{e\in [\rho]} E_e$. Then, by applying union bound, we obtain  
$$
 \mathbb{P}(E) \geq 1-\rho \exp\bigg(-\frac{L}{2e} \bigg).
$$
Let $\theta$ be a sample from $\mathcal{N}(0,1)$ and $\theta'$ be the sample from $\mathcal{N}(s\hat{\mu}_{i,s}/(s+1),1/(s+1))$. When $\hat{\mu}_{i,s}>2$,
$$
\mathbb{P}(\theta'\leq 1)\leq \frac{1}{2}\leq \mathbb{P}(\theta\leq 1).
$$
Denote $b=\mathbb{P}(\theta\leq 1)$ and $\{\textbf{a}_{1}^e,\cdots,\textbf{a}_{L}^e\}$ for the local arms with mean $X$ in group $e$. Let $\mathcal{E}_e$ be the probability that the maximum sample from $\{\textbf{a}_{1}^e,\cdots,\textbf{a}_{L}^e\}$ is larger than 1. Then,
$$
\mathbb{P}(\mathcal{E}_e \mid E )\geq 1-b^{L/{(2\sqrt{e})}}.
$$
Applying the union bound over all group $e\in [\rho]$, 
$$
\mathbb{P}(\cap_{e\in [\rho]}\mathcal{E}_e \mid E )\geq 1-\rho b^{L/{(2\sqrt{e})}}.
$$
The above inequality means that conditional on event $E$, the probability that all samples from a joint arm with mean $X$ are larger than $\rho$ is bounded by $1-\rho b^{L/{(2\sqrt{e})}}$.

By choosing $L= 2e\rho\log_{1/b} 2+2e\log_{1/b} \rho$, we have $\mathbb{P}(\cap_{e\in [\rho]}\mathcal{E}_e \mid E )\geq 1-1/2^{\rho}$, $\mathbb{P}(E)\geq 1-\exp(-\rho)$. We note that when $\theta_{\textbf{1}}\leq \rho$, $E$, and $\cap_{e\in [\rho]}\mathcal{E}_e$ happen, then we do not pull the optimal arm.

Therefore, for such a bandit instance, the probability of pulling the suboptimal joint arm before pulling the optimal joint arm is bounded by ${2^{-\rho}+e^{-\rho/2} +e^{-\rho}}$. This means that the regret incurs by pulling the suboptimal is at least $\frac{\Delta}{2^{-\rho}+e^{-\rho/2} +e^{-\rho}}$, which is exponential in $\rho$.
\end{proof}

\section{Auxiliary Lemmas}
In this section, we present some useful inequalities that are useful in our proof. 
\begin{lemma}[\citet{abramowitz1964handbook}]
\label{lem:antiGAU}
For a random variable $Z\sim \cN(\mu,\sigma^2)$,
\begin{align*}
    \frac{e^{-z^2/2}}{2}\geq  & \PP(Z>\mu+z\sigma)\geq \frac{1}{\sqrt{2\pi}}\frac{z}{z^2+1}e^{-\frac{z^2}{2}}, \notag \\
      &   \PP(Z<\mu-z\sigma) \leq  \frac{e^{-z^2/2}}{2}. 
\end{align*}
\end{lemma}
\begin{lemma}[Lemma C.3 in \citet{jin2021double}]
\label{lem:pink}
Let $N$ and $M$ be two positive integers, 
let $\gamma>0$, and let $\hat{\mu}_{n}$ be the empirical mean of $n$ random variable i.i.d from 1-subgaussian with mean $\mu$. Then for any $x\leq \mu$,
\begin{align*}
    \PP(\exists N\leq n \leq M,\hat{\mu}_{n}\leq x)\leq e^{-N(x-\mu)^2/2},
\end{align*}
and for any $x\geq \mu$,
\begin{align*}
    \PP(\exists N\leq n \leq M,\hat{\mu}_{n}\geq x)\leq e^{-N(x-\mu)^2/2}.
\end{align*}
\end{lemma}

\section{Additional Experimental Details}\label{sec:exp_details}

In this section, we provide detailed settings of the local reward distribution for Bernoulli 0101-Chain and Poisson 0101-Chain. In addition to the experimental results shown in Section \ref{sec:experiment}, we further investigate more ablation studies, including (\textit{i}): a larger number of agents and  (\textit{ii}): different $\epsilon$ condition sampling to evaluate the proposed \algname\, algorithm on Bernoulli 0101-Chain and Poisson 0101-Chain. Finally, we report our $\epsilon$-MATS against other baselines in an additional real-world experiment: wind farm control \citep{bargiacchi2018learning, verstraeten2020multiagent}.

\subsection{Reward Generation for Bernoulli and Poisson 0101-Chains}
As discussed in \Cref{sec:exp_hypergraph}, for $d=2$, the reward of a joint arm in the Bernoulli and Poisson chains depends on the local rewards for each group of agents $\{a_i,a_{i+1}\}$, $i=1,\ldots,n-1$. Each local reward follows Bernoulli or Poisson distributions  with the mean value $f(a_i,a_{i+1})$ given in \Cref{tab:bernoulli_2agents} \Cref{tab:poisson_2agents} respectively. 
It can be easily verified that the optimal joint arm for this setting is an alternating sequence of zeros and ones, starting with $0$, i.e., $(0,1,0,\ldots)$. 
For the setting where we have $d=3$ agents in each group, the reward generation is given in Table \ref{tab:bernoulli_3agents}, and the optimal joint arm is $(1,1,\ldots,1,1)$.

\begin{table}[th]
    \caption{(a) and (b): The distribution of local reward $f(a_i,a_{i+1})$ for group $i$ in the chain graph. Here $i$ is an even number. The table is transposed when $i$ is odd. There are two agents ($i, i+1)$ in each group, where $i$ is even. Each entry shows the means for each local arm of agents $i$ and $i + 1$. (c): the unscaled local reward distributions of 3-agents ($i, i+1, i+2)$ per group, where $i \mid 3$. Each entry in the first 3 columns denotes the arm and the last column indicates the corresponding success probability for each local arm of agents $i, i+1$ and $i + 2$. The table is shifted symmetrically when $i = 3c + 1$ or $i = 3c + 2$, where $c \in \mathbb{R}$.}
    \label{fig:reward_distribution_chain_graph}
    \centering
    \begin{minipage}[ht]{0.45\textwidth}
          \subtable[Bernoulli 0101-Chain, $d=2$]{
        \begin{tabular}{ccc}
         \toprule
         $f^i \sim \textit{B}$ & $a_{i+1} = 0$ & $a_{i+1} = 1$ \\
         \midrule
         $a_i = 0$ & 0.75 & 1 \\
         $a_i = 1$  & 0.25  & 0.9  \\
        \bottomrule
        \end{tabular}
        \label{tab:bernoulli_2agents}
    }
    \subtable[Poisson 0101-Chain, $d=2$]{
        \begin{tabular}{ccc}
         \toprule
         $f^i \sim \textit{P}$ & $a_{i+1} = 0$ & $a_{i+1} = 1$ \\
         \midrule
         $a_i = 0$ & 0.1 & 0.3 \\
         $a_i = 1$  & 0.2  & 0.1  \\
        \bottomrule
        \end{tabular}
        \label{tab:poisson_2agents}
    }
    \end{minipage}
    \begin{minipage}[ht]{0.45\textwidth}    
    \subtable[Bernoulli/Poisson 0101-Chain: $d=3$]{
     \begin{tabular}{cccc}
     \toprule
      $a_i$ & $a_{i+1}$ & $a_{i+2}$ & success probability\\
     \midrule
     0 & 0 & 0 & 0.5 \\
     1 & 0 & 0 & 0.9 \\
     0 & 1 & 0 & 0.8 \\
     0 & 0 & 1 & 0.2 \\
     1 & 1 & 0 & 0.6 \\
     1 & 0 & 1 & 0.3 \\
     0 & 1 & 1 & 0.4 \\
     1 & 1 & 1 & 1.0 \\
    \bottomrule
    \end{tabular} 
    \label{tab:bernoulli_3agents}
    }   
    \end{minipage}
\end{table}

\subsection{Bernoulli and Poisson 0101-Chains with 20 Agents}
Following the problem description in Section \ref{sec:exp_hypergraph}, we still consider two settings where $d = 2$ and $d = 3$ but with a larger number of agents: $m = 20$ agents in total. We demonstrate the improvement of regret performance with $\epsilon$ MATS in the Bernoulli 0101 and Poisson 0101 tasks with $20$ agents compared with MATS, MAUCE, and Random in Figure \ref{fig:large_No_agents_appendix}, indicating the consistency of benefits of lower $\epsilon$ across a different number of agents. 
\begin{figure}[thbp]
     \centering
     \subfigure[Bernoulli 0101: $m = 20$, $d = 2$]{
         \includegraphics[width=0.22\linewidth]{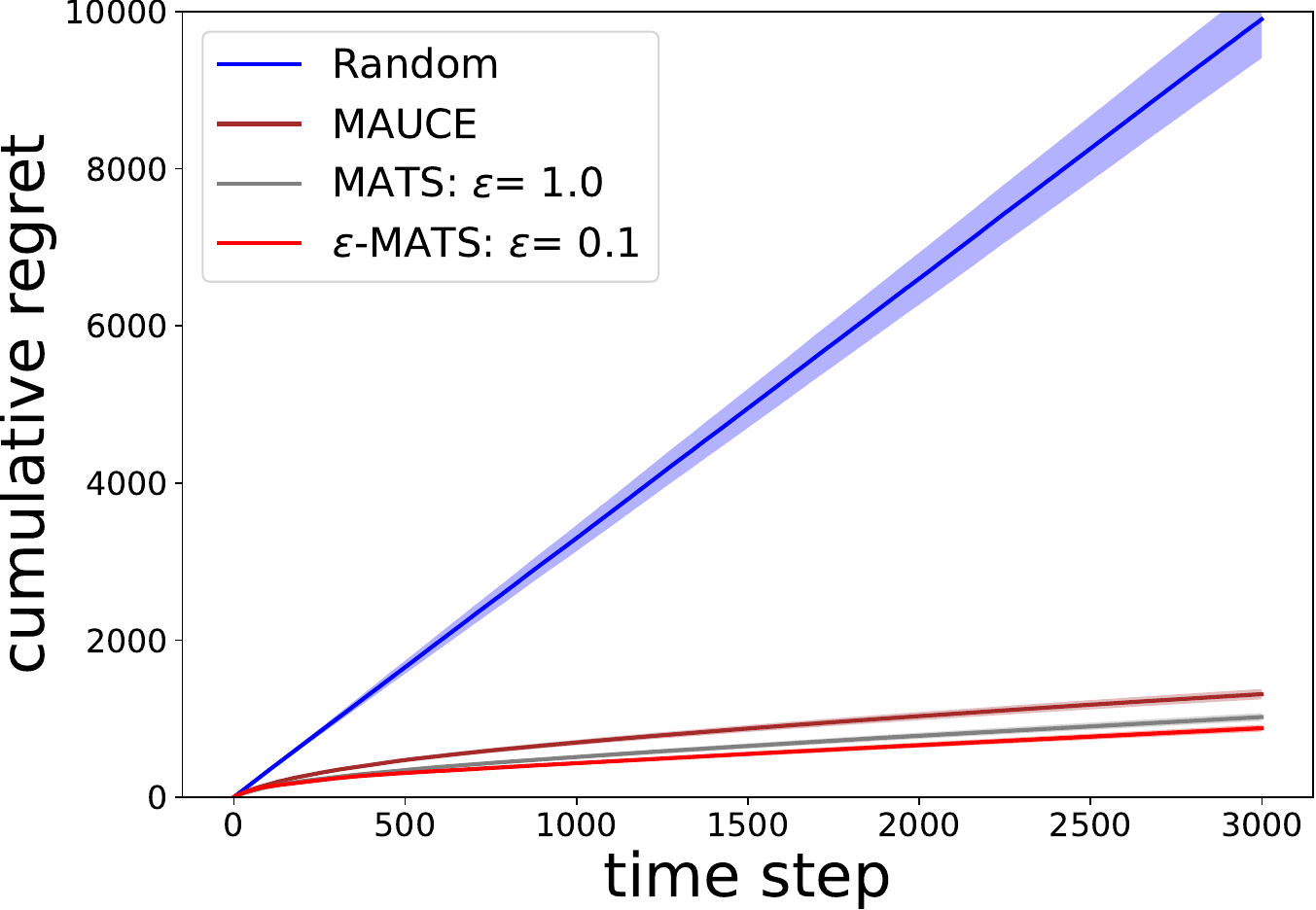}
         \label{fig:bernoulli_2agentInGroup_20agent_main}
     }
     \subfigure[Poisson 0101: $m = 20$, $d = 2$]{       
         \includegraphics[width=0.22\linewidth]{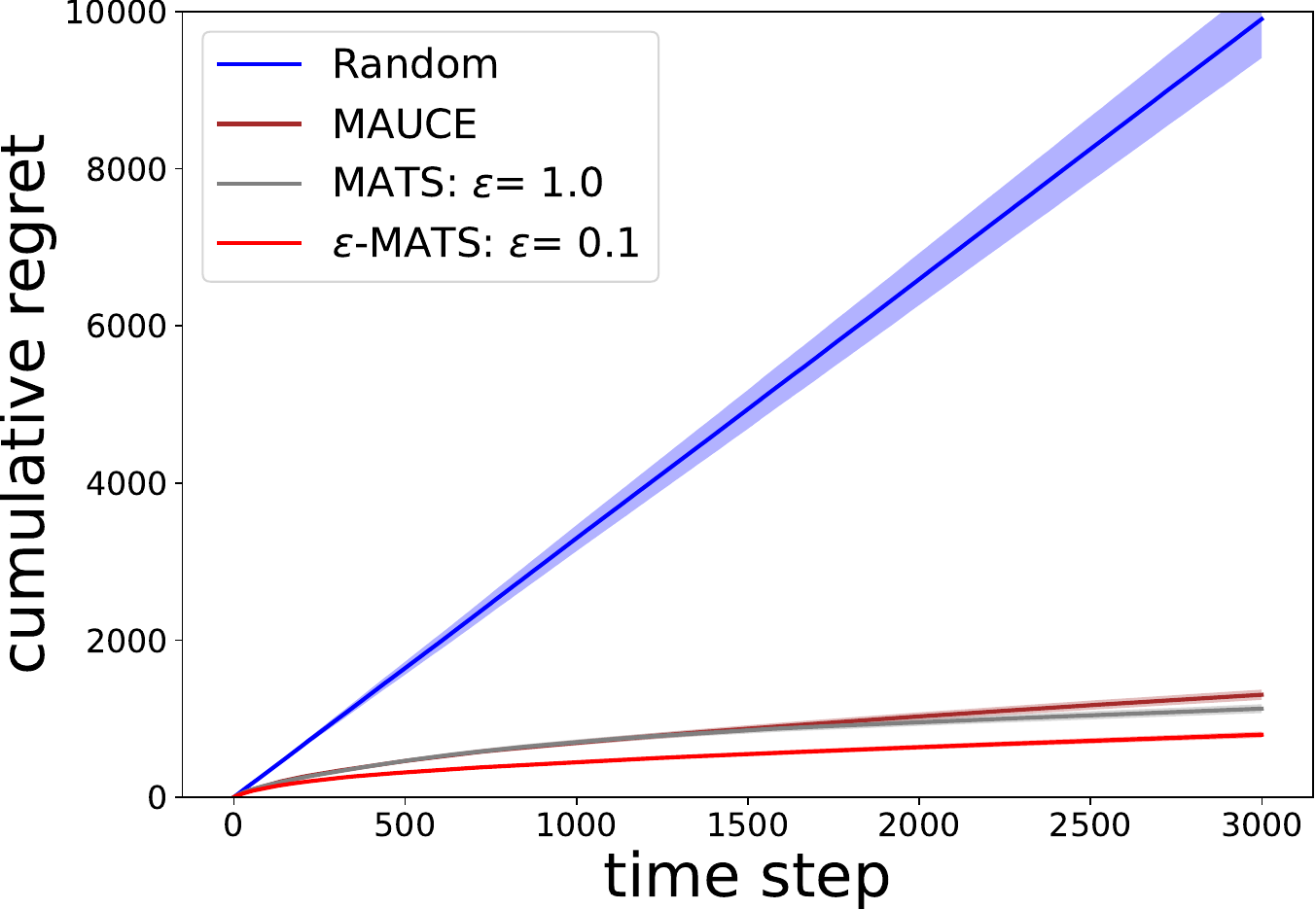}
         \label{fig:poisson_2agentInGroup_20agent_main}
     }
     \subfigure[Bernoulli 0101: $m = 20$, $d = 3$]{       
         \includegraphics[width=0.22\linewidth]{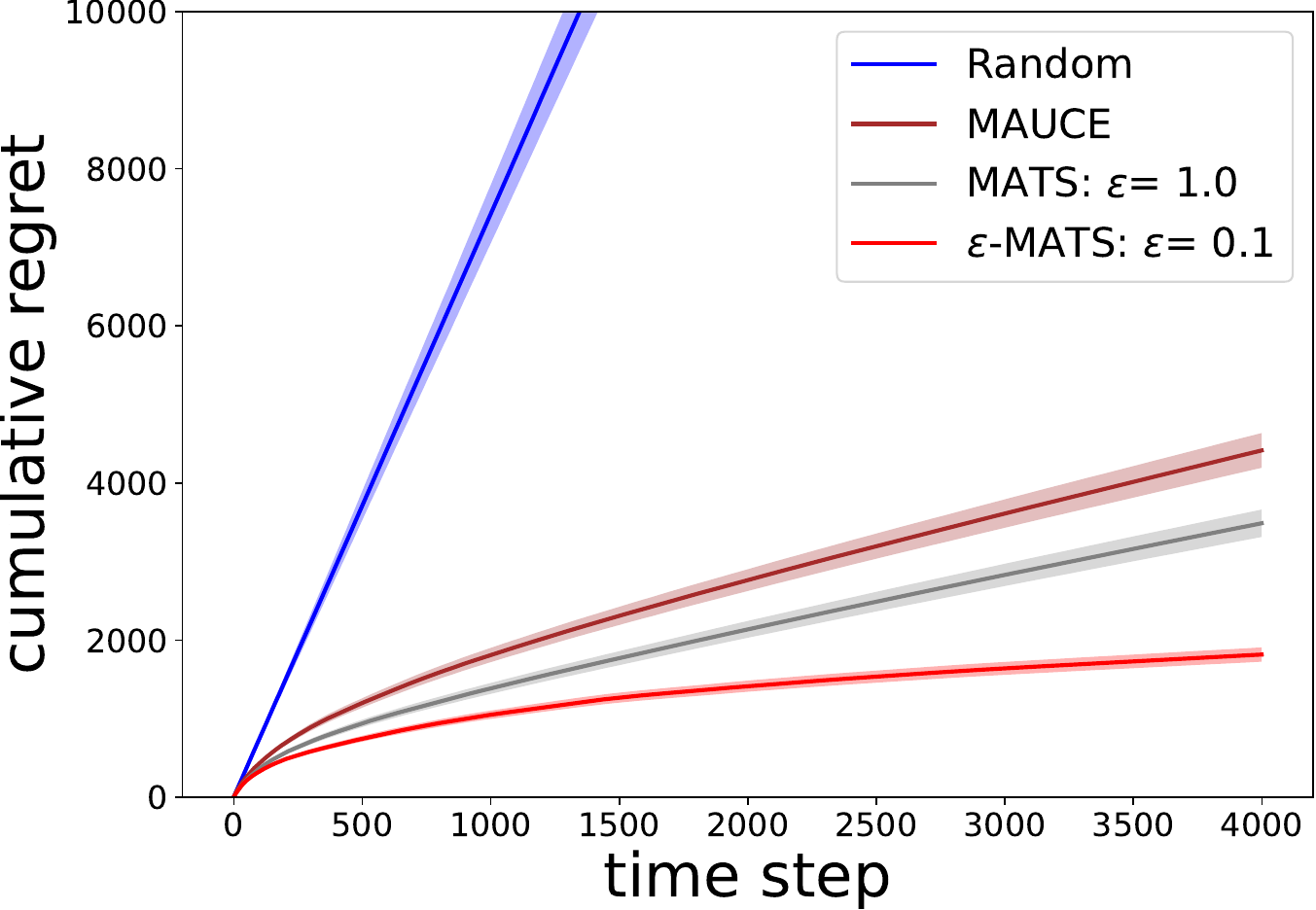}
         \label{fig:bernoulli_3agentInGroup_20agent_main}
     }
      \subfigure[Poisson 0101: $m = 20$, $d = 3$]{       
         \includegraphics[width=0.22\linewidth]{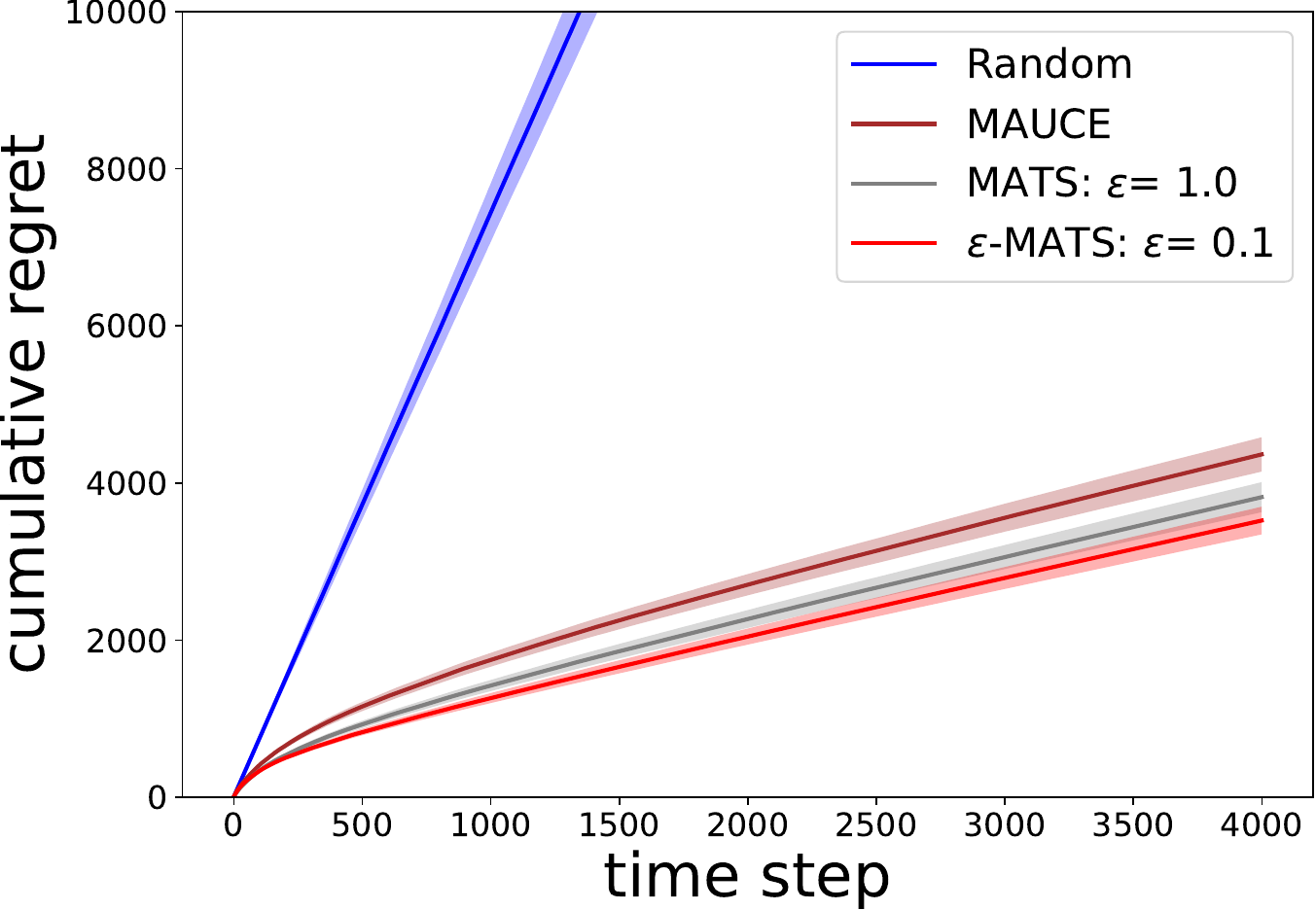}
         \label{fig:poisson_3agentInGroup_20agent_main}
     }
      \caption{Regret performance compared with other algorithm baselines in Bernoulli 0101 and Poisson 0101 with different agents in a group ($d = 2$ or $d = 3)$ with \textbf{ 20 agents} in total. %
        }\label{fig:large_No_agents_appendix}
\end{figure}

\begin{figure}[ht]
    \centering
    \subfigure[Hypergraph illustration for Gem Mining]{
        \includegraphics[width =.55\textwidth]{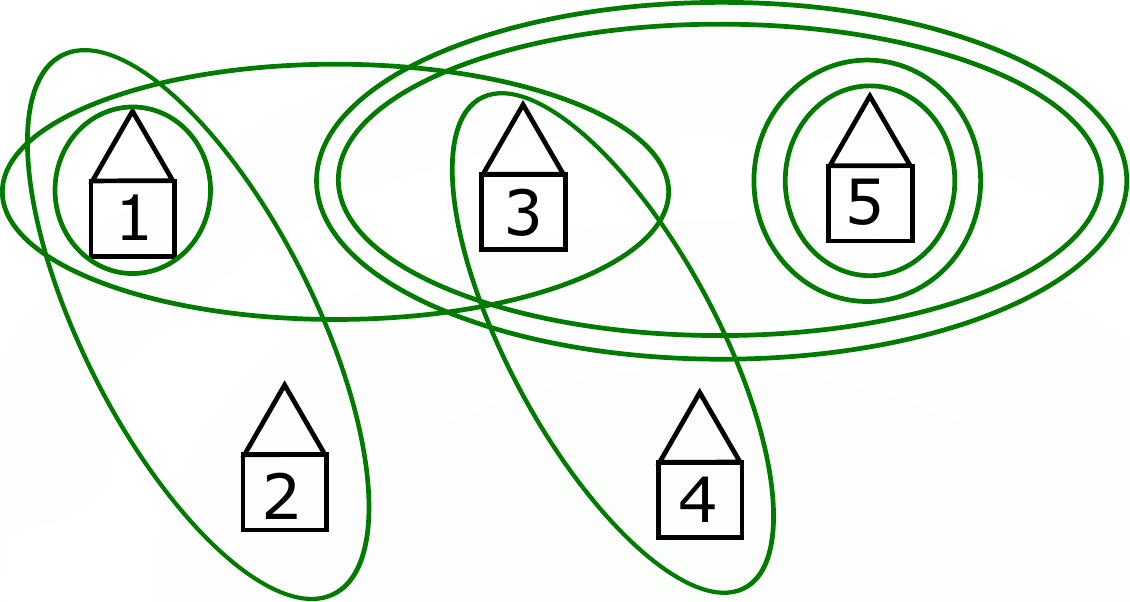}
        \label{fig:gem mine}
    }
    \subfigure[Regret comparison on Gem Mining: $m=5$]{
        \includegraphics[width=0.4\linewidth]{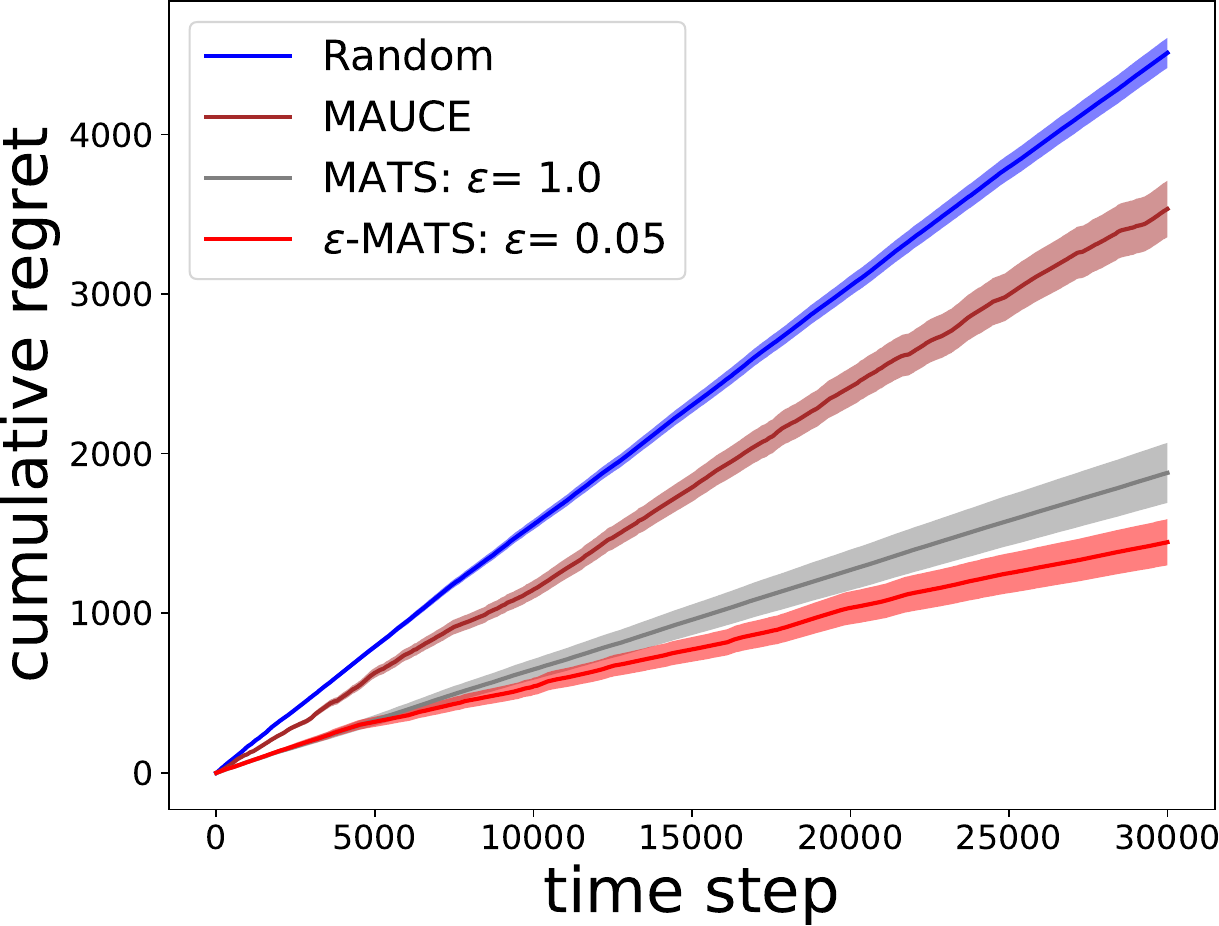}
        \label{fig:gem_mine_result}
     }
    \caption{Gem Mining experiment. (a): The hypergraph illustration. Each vertex is a village, and each hyperedge (group) corresponds to a mine. In this figure, a village is denoted as a black house while a group is denoted by a green circle. Note that there are some groups that contain the same agents (villages). (b): Regret performance compared with other algorithm baselines in Gem Mining.
    }
\end{figure}

\subsection{Gem Mining}
Now we consider a real-world application of MAMAB, e.g., the Gem Mining problem \citep{bargiacchi2018learning}, \citep{verstraeten2020multiagent}. In this problem, a mining company aims to maximize the total number of gems across multiple mines. The mine workers are represented by separate villages, each {acting as an agent in the MAMAB setup. There is only one van available per village to transport all workers from that village to one of the nearby mines. This arrangement forms a coordination graph, where each village is connected to a specific range of adjacent mines. The company must decide which mine to send its workers to, and this decision corresponds to selecting a local arm in the MAMAB framework. Each mine is formulated as a single group, and the reward for each mine follows a Bernoulli distribution. The probability of finding a gem in a mine is defined as $1.03^{w-1}p$, where $p$ represents the base probability of finding a gem (sampled uniformly from the interval [0, 0.5]), and $w$ denotes the number of workers at that mine. The cooperation among workers from multiple villages aims to maximize the probability of finding gems in a mine.
A number of workers sampled from $\{1, \ldots,5\}$ live in the same village, and each village $i$ is connected to the mine $i$ to $(i+m_i - 1)$, where $m_i$ is sampled from $[2, 4]$ while the last village is always connected to $4$ mines. We conduct the experiments with $5$ villages and $8$ mines. The diagram for this is given in Figure \ref{fig:gem mine}. The agents, in this case, correspond to the village, and the groups are determined by which villages have access to a particular gem mine. Since there are 5 villages and 8 mines, therefore there are 5 agents and 8 groups. 

Gem Mining is a more challenging problem due to the varying size of each group. We report the average cumulative regret of \algname\, against the other algorithms in Figure \ref{fig:gem_mine_result}, which again shows that \algname\, with a small $\epsilon$ outperforms all baseline methods including vanilla MATS.

\subsection{Wind Farm Control}
We conducted additional experiments on a different real-world application known as the Wind Farm Control problem \citep{bargiacchi2018learning, verstraeten2020multiagent}. Following the settings of previous studies \citep{bargiacchi2018learning, verstraeten2020multiagent}, we developed a custom simulator to replicate the energy production of an $11$-turbine wind farm (refer to Figure \ref{fig:wind_farm}). Among the turbines, $4$ downstream ones ($2, 5, 8, 11$) are not controlled by any agents, while the remaining $7$ turbines each have $3$ distinct arms (angles) representing different turning positions. In this problem, cooperation among the $7$ turbines forming $7$ groups is essential to maximize the overall power production, as turbulence generated by upstream turbines negatively affects downstream ones, with the direction of turbulence depending on the angle between each turbine and the incoming wind vector.

To simulate the wind farm, we sample wind speeds from a normal distribution at each time step. The power production of each turbine depends on the current wind speed, the turbine's angle, and the net wind direction influenced by turbulence. We normalize the overall reward by the maximum attainable reward at the highest wind strength and the minimum reward per turbine at the lowest wind strength, as per the approach in \citet{bargiacchi2018learning}. It is important to note that this measure is not an exact form of regret since choosing the optimal action does not guarantee zero regret on average.

The Wind Farm Control problem poses additional challenges due to the independent angles of the controllable agents and the fixed angles of the $4$ downstream turbines. However, even in this complex scenario, our $\epsilon$-MATS algorithm outperforms baseline methods, including MAUCE and MATS, as demonstrated by the smaller cumulative regret shown in Figure \ref{fig:wind_farm_result}.

\begin{figure}[thbp]
    \centering
    \subfigure[Wind farm setup]{
        \includegraphics[width = .45\textwidth]{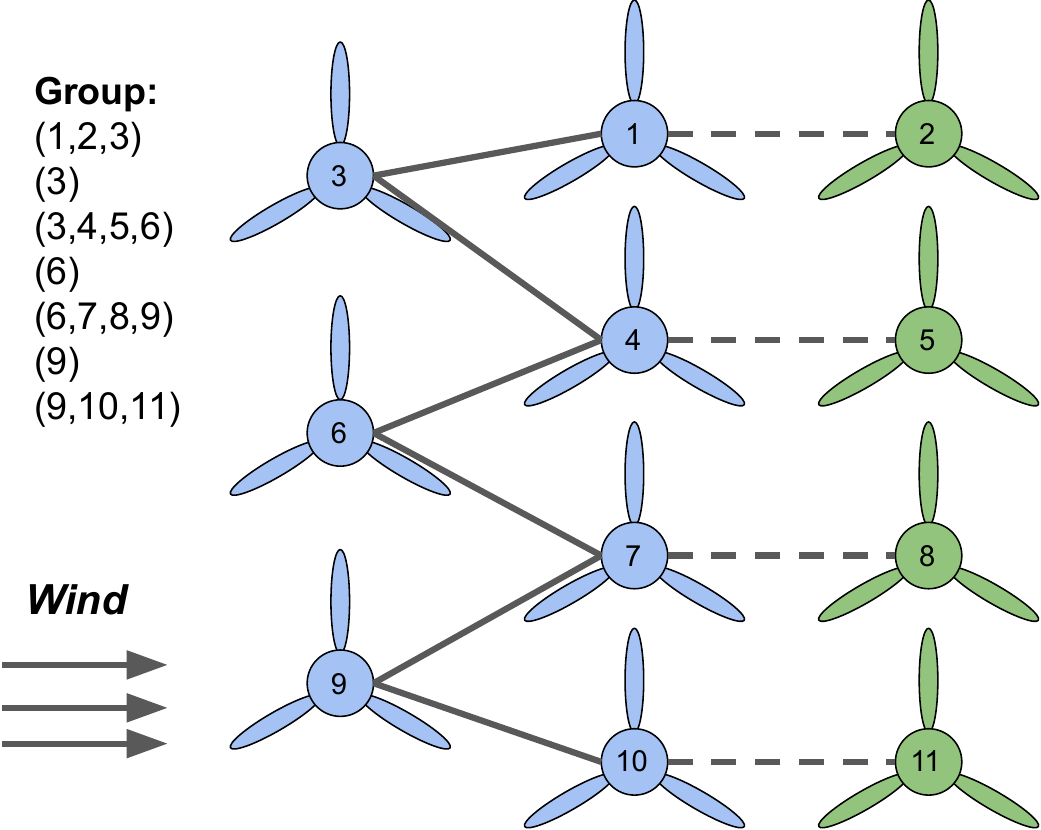}
        \label{fig:wind_farm}
    }
    \subfigure[Wind Farm Control]{
        \includegraphics[width=0.5\linewidth]{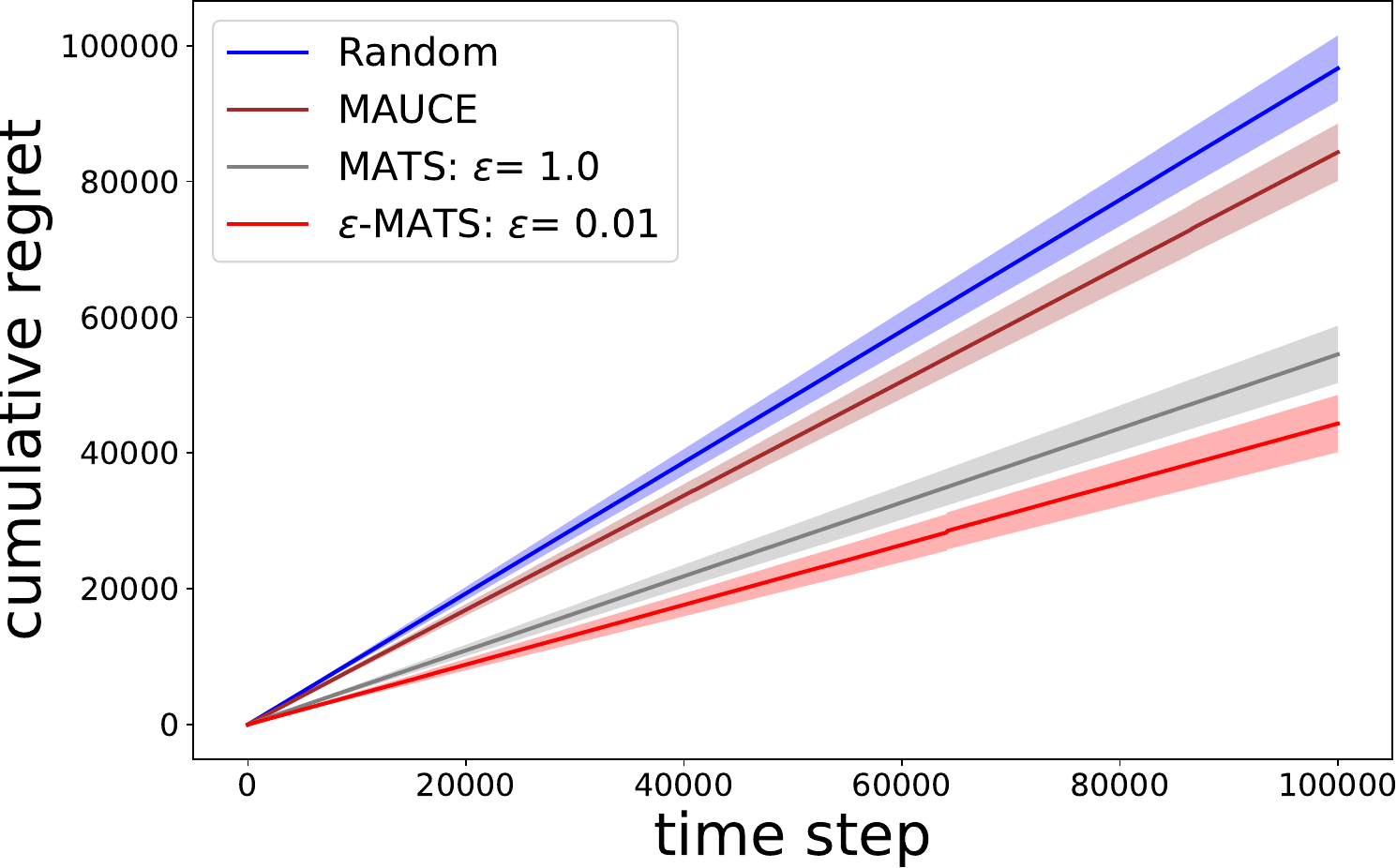}
        \label{fig:wind_farm_result}
     }
    \caption{Wind farm control experiment. (a): Wind farm setup. The incoming wind is denoted by an arrow. Agents (turbines) in a local group are listed on the left. The controllable turbines are in blue while the remaining 4 turbines are in green connected via dash lines. 
    (b): Regret performance compared with baseline algorithms in Wind Farm Control.}
\end{figure}

\end{document}